\documentclass{article} 
\usepackage{iclr2027_conference,times}

\usepackage{hyperref}
\usepackage{url}
\usepackage{mymacro}
\usepackage{array}
\usepackage{booktabs}
\usepackage{caption}
\usepackage{enumitem}

\title{Steepest Guidance: A Practical and Principled Approach to Inference-Time Alignment of Flow and Diffusion-based Models}

\author{\vspace{0cm}
Shokichi Takakura$^{1}$, Akifumi Wachi$^{1}$, Rei Higuchi$^{2,3}$, Kohei Miyaguchi$^{1}$, Taiji Suzuki$^{2,3}$ \\\vspace{0.1cm}
$^{1}$LY Corporation, $^{2}$The University of Tokyo, $^{3}$RIKEN AIP\\\vspace{0cm}
\small{\texttt{\{stakakur,akifumi.wachi\}@lycorp.co.jp, higuchi-rei714@g.ecc.u-tokyo.ac.jp}}\\\vspace{0cm}
\small{\texttt{kmiyaguc@lycorp.co.jp, taiji@mist.i.u-tokyo.ac.jp}}
}

\iclrfinalcopy 
\begin{document}

\maketitle

\begin{abstract}
   Inference-time alignment of flow and diffusion-based models
   is critical for achieving flexible generative modeling.
   Theoretically, Doob's $h$-transform provides an elegant solution to this problem,
   and most existing methods are based on this principle.
   However, in practice, estimating the optimal guidance derived from Doob's $h$-transform at inference time is challenging.
   To deal with this issue, we regard inference-time alignment as a sequential optimization problem in the space of probability measures
   and propose a novel framework called \textit{Steepest Guidance}, based on the principle of maximizing local improvement in the objective.
   We provide a theoretical analysis of the proposed method and demonstrate its effectiveness through extensive experiments.
\end{abstract}

\section{Introduction}
Flow and diffusion-based generative models have emerged as powerful tools for modeling complex data distributions across various domains,
including image generation~\citep{ho2020denoising,lipman2022flow}, natural language processing~\citep{li2022diffusion}, and molecular design~\citep{hoogeboom2022equivariant}.
In many applications, we often aim to optimize some criteria over the generated samples.
For instance, in image generation, we may want to generate images that not only look realistic but also satisfy certain aesthetic criteria or exhibit diversity.
This can be formalized as an optimization problem on a reward functional $R[\mu]$ over the distribution of generated samples $\mu$.
Typically, the reward functional is defined as the expected value of a reward function $r: \R^d \to \R$ over the distribution $\mu$, i.e., $R[\mu] = \int r(y) \mu(dy)$.
However, in some applications, we may want to consider more complex rewards such as diversity-promoting objectives~\citep{corso2023particle,vinograd2026diverse,de2025flow} or risk-sensitive objectives~\citep{zhang2020variational,de2025flow,wang2026efficient},
which cannot be expressed as the expected value of a reward function since they are non-linear in $\mu$.

Reward-guided generation is often formulated as an optimal control problem~\citep{domingo2024adjoint}, where the generative model is treated as a stochastic process, and the reward serves as a control objective.
Theoretically, the optimal solution to this problem is given by Doob's $h$-transform~\citep{rogers2000diffusions}, which provides a principled way to modify the generative process to maximize the expected reward,
while ensuring that the generated samples remain close to the original data distribution.
Several methods~\citep{domingo2024adjoint,uehara2024fine} have been proposed to address this problem by learning a guidance model that steers the generative process towards high-reward regions of the sample space.
In addition,~\citet{marion2024implicit,kawata2025direct,de2025flow} have developed a general framework for fine-tuning diffusion and flow models to maximize arbitrary reward functionals including non-linear ones.

On the other hand, a line of work has developed training-free approaches, where the guidance is estimated during the inference phase without additional training.
The plug-in estimator~\citep{dandapanthula2026we} computes the optimal guidance via backpropagation through the generative process,
which requires the reward and sampling process to be differentiable.
Recently, to deal with non-differentiable rewards,~\citet{zhu2026training} have proposed REINFORCE~\citep{williams1992simple} estimators for the optimal guidance.

In spite of the intensive research on training-free inference-time alignment of flow and diffusion models,
there are still several challenges that remain to be addressed.
First, most existing methods can only handle linear reward functionals, and cannot be applied to non-linear reward functionals.
Second, even if the reward functional is linear, the optimal guidance is difficult to estimate in practice.
Recently,~\citet{dandapanthula2026we} have revealed that the plug-in estimator with finite Monte Carlo samples is biased, leading to suboptimal performance in reward-guided generation.
In addition, as we show later, REINFORCE estimators also suffer from large bias, which causes the performance of inference-time alignment of flow models to degrade significantly.

In this paper, we regard the inference-time alignment of flow and diffusion models as a sequential optimization problem in the space of probability measures,
and propose a novel method called \textit{Steepest Guidance} based on the local improvement of the reward functional.
For linear reward functionals, we show that it can be estimated in an unbiased manner, leading to improved performance in inference-time alignment.
In addition, steepest guidance can be naturally applied to non-linear reward functionals through a particle approximation.
Our contributions are summarized as follows:
\begin{itemize}[itemsep=0.4mm,leftmargin=8mm]
    \item We regard the inference-time alignment of flow and diffusion models as a sequential optimization problem in the space of probability measures,
          and propose \textit{Steepest Guidance}, which guides the generative process towards the steepest ascent direction of the reward functional.
          Furthermore, we generalize it to an entropy-regularized objective and develop \textit{Regularized Steepest Guidance}.
    \item Accounting for the evolution of the objective, we establish provable reward improvement and global convergence of our proposed method under suitable assumptions.
          This can be seen as a functional extension of the existing analysis of Wasserstein gradient flows~\citep{bakry2014analysis,chizatmean,nitanda2022convex} from fixed-objective optimization to the optimization of objectives that evolve along the generative process.
    \item Through extensive experiments, we demonstrate that the proposed method consistently outperforms existing approaches and successfully handles non-linear reward functionals.
\end{itemize}

\vspace{-2mm}
\subsection{Related Work}
Here, we briefly review the related work and refer to Appendix~\ref{sec:related_work} for a more detailed discussion. \vspace{-1mm}

{\bf Inference-time Alignment of Flow and Diffusion Models.}
Inference-time alignment of flow and diffusion models can be categorized into two main approaches:
selection-based methods and guidance-based methods.
Selection-based methods include sequential Monte Carlo~\citep{wu2023practical},
SVDD~\citep{li2024derivative}, and Best-of-N~\citep{nakano2021webgpt}, which utilize multiple particles and select or resample them to obtain better samples.
Guidance-based methods include DOIT~\citep{zhu2026training} and Gradient Guidance~\citep{guo2024gradient},
which are often efficient since they can utilize gradient information of the reward function.
However, they require computationally expensive backpropagation or suffer from large bias. \vspace{-1mm}

{\bf Optimization in the Space of Probability Measures.}
Several works~\citep{marion2024implicit,kawata2025direct,de2025flow} have regarded reward-guided generation as an optimization problem in the space of probability measures
and developed fine-tuning methods for general reward functionals, which require additional training.
On the other hand, (Mean-field) Langevin dynamics~\citep{welling2011bayesian,chizatmean,nitanda2022convex}
is based on a Wasserstein gradient flow and optimizes a functional over the space of probability measures through gradient-based particle updates.
However, the distribution of the real-world data is often complex and multimodal,
which hinders the convergence of such methods.
Recently, Slowly Annealed Langevin Dynamics (SALD)~\citep{nitanda2026slowly} has been applied to inference-time alignment of flow and diffusion models,
but it cannot be applied to non-linear reward functionals. \vspace{-2mm}
\section{Preliminaries}
\vspace{-2mm}
In this section, we introduce the flow and diffusion-based generative models and formulate the inference-time alignment as an optimal control problem.
For $d \in \N$, let $\mathcal{P}$ be the space of probability measures on $\R^d$ which have density functions, and finite entropy and second moment.
We denote a data distribution by $\pi_1 \in \mathcal{P}$.

\subsection{Flow and Diffusion-based Models}
A flow matching model is formulated as the following ordinary differential equation (ODE):
\begin{align*}
    \d Y_t = v_t(Y_t) \d t, \quad Y_0 \sim \mathcal{N}(0, I),
\end{align*}
where $v_t(Y_t)$ is a time-dependent vector field, which is approximated by a neural network.
Typically, $v_t(x)$ is defined as $\Expec{\odv{Y_t}{t} \mid Y_t = x}$,
where $Y_t := (1-t) Y_0 + t Y_1$, $Y_0 \sim \mathcal{N}(0, I)$, and $Y_1 \sim \pi_1$.
Furthermore, for any $\sigma_t \in \R$, the solution of the SDE
\begin{align}
    \d Y_t = \ab(v_t(Y_t) + \frac{\sigma_t^2}{2} \grad \log \pi_t(Y_t)) \d t + \sigma_t \d W_t, \quad Y_0 \sim \mathcal{N}(0, I)~\label{eq:flow_sde}
\end{align}
has the same marginal distribution as the solution of the ODE, where $\pi_t$ is the marginal distribution of $Y_t$.
In this paper, we consider a memoryless noise schedule $\sigma_t^2 = 2(1-t)/t$ following~\citet{domingo2024adjoint,bergmeister2026reinforce}.

On the other hand, for diffusion models, we first define the following (forward) SDE~\citep{jiao2025towards}:
\begin{align*}
    \d X_t = -\frac{1}{2(1-t)}X_t \d t + \frac{1}{\sqrt{1-t}} \d W_t, \quad X_0 \sim \pi_1.
\end{align*}
Then, the reverse-time SDE is given by
\begin{align*}
    \d Y_t = \ab(\frac{1}{2}Y_t + \grad \log \pi_t(Y_t)) \frac{\d t}{t} + \frac{1}{\sqrt{t}} \d W_t, \quad Y_0 \sim \mathcal{N}(0, I),
\end{align*}
where $Y_t$ has the same marginal distribution as $X_{1-t}$. Here, $\grad \log \pi_t(y)$ is the score function of $\pi_t$ and is approximated by a neural network.

In this paper, we consider the following SDE as a general form of flow and diffusion models:
\begin{align*}
    \d Y_t = b_t(Y_t) \d t + \sigma_t \d W_t, \quad Y_\varepsilon \sim \pi_\varepsilon, \quad t \in [\varepsilon,1].
\end{align*}
Here, we introduce a small constant $\varepsilon \geq 0$ and assume that $\varepsilon > 0$ for theoretical analysis to avoid blow-up of $\sigma_t$ at $t=0$.
We denote its path measure by $\mathbb{P}_{\pi}$ and the marginal distribution of $Y_t$ by $\pi_t$.
In the following, we assume that the second moment of $\pi_t$ is finite.

\subsection{Inference-time Alignment of Flow Models}
In this paper, we aim to maximize a reward functional $R: \mathcal{P} \to \R$ over the generated samples
by adding a guidance term $g_t(y)$ to the generative process:
\begin{align*}
    \d Y_t = (b_t(Y_t) + g_t(Y_t)) \d t + \sigma_t \d W_t, \quad Y_\varepsilon \sim \pi_\varepsilon, \quad t \in [\varepsilon,1].
\end{align*}
This guidance based formula has been widely employed by several existing works
such as classifier/classifier-free guidance \citep{dhariwal2021diffusion,ho2022classifier}.
We call a guidance $g_t(y)$ admissible if it satisfies
$
    \int_\varepsilon^1 \Expec[\pi_t^g]{\norm{g_t(Y_t)}^2 / \sigma_t^2} \d t < \infty.
$
We denote its path measure by $\mathbb{P}_{\pi^g}$ and the marginal distribution of $Y_t$ by $\pi_t^g$.
Then, we consider the following optimization problem:
\begin{align}
    \max_{g} R[\pi^g_1] - \frac{1}{\lambda} \kl(\mathbb{P}_{\pi^g} \mid \mathbb{P}_{\pi}), \label{eq:objective}
\end{align}
where $\kl$ is the KL divergence and $\lambda > 0$ is a regularization parameter.
Typically, the reward functional is defined as the expected value of a reward function $r: \R^d \to \R$ over the distribution $\mu$, i.e., $R[\mu] = \int r(y) \mu(dy)$.
On the other hand, in some applications, we may want to consider more complex rewards such as risk-averse objectives like CVaR or diversity-promoting objectives like Rao's quadratic entropy:
\begin{align*}
    R_{\mathrm{CVaR}}[\mu] = \Expec[\mu]{r(Y) \mid r(Y) \leq q_\alpha}, \quad R_{\mathrm{Rao}}[\mu] = -\frac{1}{2}\int \int k(y, y') \mu(dy) \mu(dy'),
\end{align*}
where $q_\alpha$ is the $\alpha$-quantile of $r(Y)$ and $k: \R^d \times \R^d \to \R$ is a positive definite kernel.
See Table~\ref{tab:reward_functionals} for other examples of (non-linear) reward functionals.

In this paper, we assume that the reward functional $R$ has a functional derivative $\fdv{R}{\mu}[\mu](y)$ with $\|\grad \fdv{R}{\mu}[\mu](y)\|,|\fdv{R}{\mu}[\mu](y)| \leq C_V$
and $\mu_1^* := \argmax_{\mu_1 \in \mathcal{P}} R[\mu_1] - \frac{1}{\lambda} \kl(\mu_1 \mid \pi_1)$ exists.
\begin{definition}[Functional Derivative]
    The functional derivative of a reward functional $R: \mathcal{P} \to \R$ with respect to a measure $\mu \in \mathcal{P}$ is a function $\fdv{R}{\mu}[\mu](y)$ such that for any $\nu \in \mathcal{P}$ and the mixture path $\mu_\epsilon := (1-\epsilon)\mu + \epsilon\nu$, it holds that
    $
        \frac{\d}{\d \epsilon} R[\mu_\epsilon] \Big|_{\epsilon=0} = \int \fdv{R}{\mu}[\mu](y) \, (\nu-\mu)(\d y).
    $
\end{definition}

\subsection{Optimal Control via Doob's $h$-transform}
As a theoretical principle to determine the guidance for the reward maximization, Doob's $h$-transform has been utilized by several existing works~\citep{uehara2024fine,domingo2024adjoint,kawata2025direct,bergmeister2026reinforce,zhu2026training} to guide the generative process.
Formally, Doob's h-function is defined as $h_t(y) = \Expec{\exp(\lambda \fdv{R}{\mu}[\mu_1^*](Y_1)) \mid Y_t = y}$, and also it can be defined through the following SDE in the sense that the density ratio between the marginal distributions of $Y_t$ for the original process and the following process is proportional to $h_t$:
\begin{align*}
    \d Y_t = (b_t(Y_t) + g_t^*(Y_t)) \d t + \sigma_t \d W_t, \quad Y_0 \sim \mathcal{N}(0, I),
\end{align*}
where $g_t^*(y) = \sigma_t^2 \grad \log h_t(y)$.
As discussed in~\citet{domingo2024adjoint}, under the memoryless noise schedule this Doob transform solves the control problem~\eqref{eq:objective} with $\varepsilon = 0$.

Unfortunately, Doob's $h$-transform does not admit a closed-form expression and requires numerical estimation, presenting a significant computational challenge.
To tackle this problem, some recent works have proposed to estimate the optimal guidance $g_t^*(y)$ at inference time using Monte Carlo samples from the conditional distribution $\pi_1(\cdot \mid Y_t = y)$.
From the chain rule, the optimal guidance $g_t^*(y)$ can be expressed as
\begin{align*}
    g_t^*(y) = \sigma_t^2
    \frac{\grad_y \Expec[\pi_1(\cdot \mid Y_t = y)]{\exp(\lambda \fdv{R}{\mu}[\mu_1^*](Y_1))}}{\Expec[\pi_1(\cdot \mid Y_t = y)]{\exp(\lambda \fdv{R}{\mu}[\mu_1^*](Y_1))}}.
\end{align*}
Let $z_1, \dots, z_k$ be i.i.d. samples from the conditional distribution $\pi_1(Z \mid Y_t = y)$.
Then, the denominator can be estimated as $\frac{1}{k} \sum_{i=1}^k \exp(\lambda \fdv{R}{\mu}[\mu_1^*](z_i))$.
To estimate the numerator, there are two main approaches: Plug-in and REINFORCE estimators.

\paragraph{Plug-in Estimator.}
Assume that the sampling process can be expressed as $z_i = f(y, \epsilon_i)$ for some deterministic function $f$ and random variable $\epsilon_i$.
Then, the plug-in estimator of $g_t^*(y)$ is given by
\begin{align*}
    \grad_y \Expec[\pi_1(\cdot \mid Y_t = y)]{\exp\left(\lambda \fdv{R}{\mu}[\mu_1^*](Y_1)\right)} \simeq \frac{1}{k} \sum_{i=1}^k \grad_y \exp\left(\lambda \fdv{R}{\mu}[\mu_1^*](f(y, \epsilon_i))\right).
\end{align*}
Note that the plug-in estimator requires the reward and sampling process to be differentiable.
Furthermore, even if the reward and sampling process are differentiable, computing the gradient is computationally expensive since it requires backpropagation through the generative process.

\paragraph{REINFORCE Estimator.}
Following \citet{zhu2026training}, let us consider the following identity~\citep{williams1992simple}:
\begin{align*}
    \grad_y \Expec[\pi_1(\cdot \mid Y_t = y)]{f(Y_1)} = \Expec[\pi_1(\cdot \mid Y_t = y)]{f(Y_1) \grad_y \log \pi_1(Y_1 \mid Y_t = y)}.
\end{align*}
Then, the REINFORCE estimator of $g_t^*(y)$ is given by
\begin{align*}
    \grad_y \Expec[\pi_1(\cdot \mid Y_t = y)]{\exp\left(\lambda \fdv{R}{\mu}[\mu_1^*](Y_1)\right)} \simeq \frac{1}{k} \sum_{i=1}^k \exp\left(\lambda \fdv{R}{\mu}[\mu_1^*](z_i)\right) \grad_y \log \pi_1(z_i \mid Y_t = y).
\end{align*}
Instead of the differentiability of the reward and sampling process, the REINFORCE estimator only requires the differentiability of the log-likelihood of the conditional distribution $\pi_1(\cdot \mid Y_t = y)$.
Therefore, the REINFORCE estimator can be applied to non-differentiable reward functions.
As shown in the following lemma, we can compute the gradient of the log-likelihood of the conditional distribution $\pi_1(\cdot \mid Y_t = y)$ from the score function of the marginal distribution $\pi_t$.
\begin{lemma}\label{lem:conditional_score}
    For any $t \in [\varepsilon, 1)$, we have
    \begin{align*}
        \grad_y \log \pi_1(z \mid Y_t = y) =
        \begin{cases}
            -\frac{1}{(1-t)^2} (y - tz) -\frac{tv_t(y) - y}{1-t}       & (\text{Flow Matching Model}), \\
            -\frac{1}{1-t} \cdot (y - \sqrt{t}z) - \grad \log \pi_t(y) & (\text{Diffusion Model}).
        \end{cases}
    \end{align*}
\end{lemma}
See Appendix~\ref{proof:conditional_score} for the proof.

\paragraph{Challenges in Estimating the Optimal Guidance.}
While Doob's $h$-transform provides a principled way to compute the optimal guidance,
there are two main challenges in computing the optimal guidance $g_t^*(y)$.

\emph{Non-linearity of the reward functional:}
First, while $\fdv{R}{\mu}[\mu_1^*](y) = r(y)$ for linear reward functionals, we cannot compute $\fdv{R}{\mu}[\mu_1^*](y)$ for general non-linear reward functionals
because it requires the knowledge of the optimal distribution $\mu_1^*$, which is exactly what we are looking for.

\begin{wrapfigure}[15]{r}{0.33\textwidth}
    \vspace{-4mm}
    \centering
    \includegraphics[width=\linewidth]{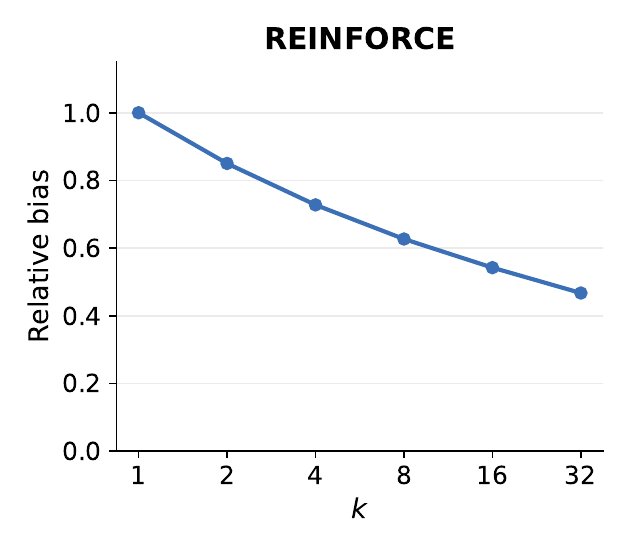}\vspace{-2mm}
    \caption{Relative bias of the REINFORCE estimators for the toy Gaussian problem.}
    \label{fig:bias}
\end{wrapfigure}
\emph{Non-linearity of the guidance:}
Second, even if we can compute $\fdv{R}{\mu}[\mu_1^*]$, we cannot obtain unbiased estimates of $g_t^*(y)$ since a non-linear transformation (logarithm) is applied to the expectation.
Therefore, if we use a finite number of samples, the estimator of $g_t^*(y)$ is biased.
For the plug-in estimator,~\citet{dandapanthula2026we} have shown that the bias of the plug-in estimator leads to suboptimal performance in reward-guided generation.
Here, we show that the REINFORCE estimator may be biased even in the case where the plug-in estimator is unbiased.
Specifically, we consider a simple toy problem, where the terminal distribution is the one-dimensional Gaussian and the reward function is $r(y) = y$.
In this case, the plug-in estimator is unbiased
but the REINFORCE estimator suffers from large bias as shown in Fig.~\ref{fig:bias}.
See Appendix~\ref{sec:experimental_details} for detailed experimental settings.
\vspace{-2mm}
\section{Proposed Method: Steepest Guidance}
\vspace{-2mm}
As shown in the previous section, the optimal guidance $g_t^*(y)$ is difficult to estimate in practice due to the non-linearity of the reward functional and the guidance.
Instead of estimating the globally optimal guidance, we propose to consider the local optimality of the guidance
and design a \textit{steepest guidance} that improves the reward functional in a local manner.
The following proposition is fundamental to analyzing the improvement of the reward functional with respect to the guidance.
\begin{proposition}\label{prop:reward_improvement}
    For any admissible guidance $g_t(y)$, we have
    \begin{align*}
        R[\pi_1^g] \!-\! R[\pi_1] \!=\! \int_{\varepsilon}^1 \! \Expec[\pi_t^g]{g_t(Y_t) \cdot \grad \fdv{V(t, \pi_t^g)}{\mu}(Y_t)} \d t,~
        \kl(\mathbb{P}_{\pi^g} \mid \mathbb{P}_{\pi}) \!=\! \int_{\varepsilon}^1 \!\Expec[\pi_t^g]{\frac{\|g_t(Y_t)\|^2}{2\sigma_t^2}} \d t,
    \end{align*}
    where $V(t, \mu) := R[K_t \mu]$ and $K_t$ is the transition kernel from $t$ to $1$.
\end{proposition}
See Appendix~\ref{proof:reward_improvement} for the proof.
We observe that this proposition provides the {\it steepest direction} to improve the objective at each time $t$,
and that the functional derivative $\fdv{V(t, \pi_t^g)}{\mu}(Y_t)$ directly appears in the expectation without any non-linear transformation,
in contrast to Doob's $h$-transform.
Similar results can be found for linear functionals in \citet{jiao2025towards}
but we extend them to general reward functionals through a generator-based argument that extends to non-linear functionals.

\vspace{-2mm}
\subsection{Local Improvement of the Reward Functional}
\vspace{-2mm}
Here, we explain the intuitive idea behind our proposed method.
Let $\mu_\tau$ be the marginal distribution of $Y_\tau$ at time $\tau \in [\varepsilon, 1)$.
Then, we consider the following SDE, in which a time-independent guidance is applied over a small interval $[\tau, \tau + \Delta\tau]$ with $\tau + \Delta\tau \leq 1$
\begin{align*}
    \d Y_t = (b_t(Y_t) + g(Y_t))\cdot \d t + \sigma_t \d W_t \quad (t \in [\tau, \tau + \Delta \tau]), \\
    \d Y_t = b_t(Y_t)\cdot \d t + \sigma_t \d W_t \quad (t \in [\tau + \Delta \tau, 1]).
\end{align*}
Intuitively, Proposition \ref{prop:reward_improvement} implies that for a small time interval $\Delta \tau$, we have
\begin{align*}
    \textstyle  R[\mu_1^g] - R[\mu_1] - \frac{1}{\lambda} \kl(\mathbb{P}_{\mu^g} \mid \mathbb{P}_{\mu}) \simeq \Delta\tau \cdot \Expec[\mu_\tau]{g(Y_\tau) \cdot \grad_y \fdv{V(\tau, \mu_\tau)}{\mu}(Y_\tau) - \frac{\|g(Y_\tau)\|^2}{2\lambda\sigma_\tau^2}},
\end{align*}
which is a quadratic function of $g$ and can be maximized by setting $g(y) = \lambda \sigma_\tau^2 \grad_y \fdv{V(\tau, \mu_\tau)}{\mu}(y)$.

Based on the above observation, we propose a \textit{steepest guidance} defined as
\begin{align}
    g_t(y) = \lambda \sigma_t^2 \grad_y \fdv{V(t, \pi_t^g)}{\mu}(y), \label{eq:steepest_guidance}
\end{align}
which is characterized by the local optimality of the reward improvement.
The following theorem shows that the steepest guidance provably improves the reward functional
while controlling the KL divergence between the original and guided processes.
\begin{theorem}\label{thm:improvement_steepest}
    Assume that the steepest guidance defined in Eq.~\eqref{eq:steepest_guidance} is admissible. Then, we have
    \begin{align*}
        \textstyle    R[\pi_1^g] - R[\pi_1] & =\textstyle \lambda \Gamma_1, \quad \kl(\mathbb{P}_{\pi^g} \mid \mathbb{P}_{\pi}) = \frac{\lambda^2 \Gamma_1}{2},
    \end{align*}
    where $\Gamma_t := \int_{\varepsilon}^t \Expec[\pi_t^g]{\sigma_t^2 \left\|\grad_y \fdv{V(t, \pi_t^g)}{\mu}(Y_t)\right\|^2} dt$ for $t \in [\varepsilon, 1]$.
\end{theorem}
See Appendix~\ref{proof:improvement_steepest} for the proof.
We may interpret this theorem as a functional version of the descent lemma, analogous to the standard analysis of finite-dimensional gradient descent.
It also suggests that a larger $\lambda$ is expected to yield a greater improvement in the reward functional.
Indeed, in Section~\ref{sec:ConvergenceAnalysis}, we establish global convergence for a regularized variant of steepest guidance
by taking sufficiently large $\lambda$, under suitable assumptions.

\vspace{-2mm}
\subsection{Practical Implementation of Steepest Guidance}
\vspace{-2mm}
To implement the steepest guidance in practice, we need to discretize the process and estimate $g_t(y)$ from finite Monte Carlo samples.
First, we focus on the case where $R$ is linear in $\mu$, i.e., $R[\mu] = \int r(y) \mu(dy)$ for some reward function $r: \R^d \to \R$.
In this case, we have $V(t, \mu) = \int r(y_1) K_t \mu(dy_1)$ and $\fdv{V(t, \mu)}{\mu}(y) = E[r(Y_1) \mid Y_t = y]$.
Therefore, the steepest guidance can be estimated as follows:
\begin{align}
    \hat{g}_t^{\mathrm{steepest}}(y) = \lambda \sigma_t^2 \frac{1}{k} \sum_{i=1}^k r(y_i) \grad_y \log \pi_1(y_i \mid Y_t = y), \label{eq:steepest_estimator}
\end{align}
where $y_i$ are i.i.d. samples from the conditional distribution $\pi_1(\cdot \mid Y_t = y)$.
An important property of the above estimator is that it is unbiased, i.e., $\Expec{\hat{g}_t^{\mathrm{steepest}}(y)} = g_t(y)$,
which is in contrast to the case of the optimal guidance.

\begin{remark}
    In practice, it is often computationally expensive to sample from the conditional distribution $\pi_1(\cdot \mid Y_t = y)$.
    Thus, previous works have proposed to use the approximate sampling processes such as GLASS flow~\citep{holderrieth2025glass} and Diamond Map~\citep{holderrieth2026diamond}.
    Under the approximate sampling process, the estimator in Eq.~\eqref{eq:steepest_estimator} is biased but can still be effective in practice, as shown in our experiments.
\end{remark}

In the case where $R$ is non-linear in $\mu$, the population first variation $\fdv{R}{\mu}[\mu_{1,t}](y)$ depends on the unknown terminal distribution $\mu_{1,t} := K_t\pi_t^g$.
Following previous work~\citep{takakura2026inference}, we approximate this first variation using lookahead particles.
Let $\{y^{(j)}\}_{j=1}^N$ be samples in a batch. Then, for each $j = 1, \dots, N$, we generate $k$ i.i.d. samples $\{y_i^{(j)}\}_{i=1}^k$ from the conditional distribution $\pi_1(\cdot \mid Y_t = y^{(j)})$
and define the empirical measure $\hat{\mu}_1 = \frac{1}{Nk} \sum_{i=1}^k \sum_{j=1}^N \delta_{y_i^{(j)}}$.
With a slight abuse of notation, we write $\fdv{R}{\mu}[\hat{\mu}_1](y)$ for the sample-based plug-in approximation obtained
by replacing the population quantities in $\fdv{R}{\mu}[\mu_{1,t}](y)$ with their empirical counterparts.
We then estimate the steepest guidance by
\begin{align}
    \hat{g}_t^{\mathrm{steepest}}(y^{(j)}) = \lambda \sigma_t^2 \frac{1}{k} \sum_{i=1}^k \fdv{R}{\mu}[\hat{\mu}_1](y_i^{(j)}) \grad_y \log \pi_1(y_i^{(j)} \mid Y_t = y^{(j)}). \label{eq:steepest_estimator_nonlinear}
\end{align}
Unlike the linear-reward estimator in Eq.~\eqref{eq:steepest_estimator}, this sample-based estimator is generally biased for finite $N$ because the estimated first variation depends non-linearly on the empirical samples.
We show the detailed algorithm of steepest guidance in Algorithm~\ref{alg:steepest-guidance}.

\section{Regularized Steepest Guidance}
We are often interested in optimizing the terminal KL-regularized objective instead of the original reward functional $R$
to ensure that the generated samples are close to the original distribution:
\begin{align}
    \max_{g} J_\eta(\pi^g_1) := \max_{g} R[\pi^g_1] - \frac{1}{\eta} \kl(\pi^g_1 \mid \pi_1), \label{eq:objective_kl_regularized}
\end{align}
for $\eta > 0$.
From the data-processing inequality, we have $\kl(\pi^g_1 \mid \pi_1) \leq \kl(\mathbb{P}_{\pi^g} \mid \mathbb{P}_{\pi})$.
Therefore, we can obtain the following results:
\begin{corollary}\label{cor:improvement_steepest_kl}
    Applying steepest guidance $g_t(y) = \lambda \sigma_t^2 \nabla_y \fdv{V(t, \pi_t^g)}{\mu}(y)$ with $0 < \lambda \leq 2\eta$,
    we have
    \begin{align*}
        J_\eta(\pi^g_1) - J_\eta(\pi_1)
        \geq \lambda \left(1 - \frac{\lambda}{2\eta}\right) \int_{\varepsilon}^1 \Expec[\pi_t^g]{\sigma_t^2 \left\|\grad_y \fdv{V(t, \pi_t^g)}{\mu}(Y_t)\right\|^2} \d t \geq 0.
    \end{align*}
\end{corollary}
See Appendix~\ref{proof:improvement_steepest_kl} for the proof.
This corollary shows that the steepest guidance improves the KL-regularized objective for $0 < \lambda \leq 2\eta$.
However, this guarantee constrains the guidance strength relative to the KL regularization parameter,
and thus $\lambda$ cannot be chosen independently of $\eta$.

Instead of controlling the terminal KL divergence via pathwise KL divergence,
we can consider the terminal KL-regularized objective $J_\eta(K_t\mu)$ directly
thanks to our general formulation.
However, computing the first-order variation of $J_\eta(K_t\mu)$ requires the density ratio $\frac{K_t \mu(y)}{K_t \pi_t(y)}$,
which is generally intractable.
To deal with this issue, let us consider the quantity $\mathcal{V}(t, \mu) := V(t, \mu) - \frac{1}{\eta} \kl(\mu \mid \pi_t)$.
From the data-processing inequality, we can show that $\mathcal{V}(t, \mu)$ is a lower bound on $J_\eta(K_t\mu)$.
Based on this observation, we propose to use $\mathcal{V}(t, \mu)$ to construct the following guided SDE:
\begin{align}
    \d Y_t^g & = (b_t(Y_t^g) + g_t(Y_t^g))\cdot \d t + \sigma_t \d W_t, \label{eq:original_sde}
\end{align}
where $$g_t(y)   = \lambda \sigma_t^2 \nabla_y \fdv{\mathcal{V}(t, \pi_t^g)}{\mu}(y) = \lambda \sigma_t^2 \nabla_y \fdv{V(t, \pi_t^g)}{\mu}(y) - \frac{\lambda}{\eta} \sigma_t^2 \nabla_y \log \frac{\pi_t^g(y)}{\pi_t(y)}.$$
Even in this case, the steepest guidance requires the density ratio
but we can construct an SDE that has the same marginal distribution as Eq.~\eqref{eq:original_sde} without computing it.
\begin{proposition}\label{prop:equivalence_sde}
    The following SDE has the same marginal distribution as Eq.~\eqref{eq:original_sde}:
    \begin{align*}
         & ~~ \d Y_t^g   = (b_t(Y_t^g) + g_t(Y_t^g))\cdot \d t + \sqrt{1 + \tfrac{2\lambda}{\eta}} \cdot \sigma_t \cdot \d W_t,                           \\
         & \text{where}~~~ g_t(y)    := \lambda \sigma_t^2 \nabla_y \fdv{V(t, \pi_t^g)}{\mu}(y) + \tfrac{\lambda}{\eta} \sigma_t^2 \grad_y \log \pi_t(y).
    \end{align*}
\end{proposition}\vspace{-1mm}
See Appendix~\ref{proof:equivalence_sde} for the proof.
We call the above method \textit{regularized} steepest guidance.
As shown in the following theorem, regularized steepest guidance improves $J_\eta$.
\begin{theorem}\label{thm:improvement_steepest_terminal_kl}
    For the regularized steepest guidance, we have
    \begin{align*}
        J_\eta(\pi_1^g) - J_\eta(\pi_1)
        \geq \lambda \int_\varepsilon^1 \Expec[\pi_t^g]{\sigma_t^2 \left\|\nabla \frac{\delta \mathcal{V}(t, \pi_t^g)}{\delta \mu}(Y_t)\right\|^2} \d t.
    \end{align*}
\end{theorem}\vspace{-1mm}
See Appendix~\ref{proof:improvement_steepest_terminal_kl} for the proof.
In contrast to Corollary~\ref{cor:improvement_steepest_kl},
the above theorem ensures the improvement of the KL-regularized objective even when $\lambda > 2\eta$.
\vspace{-1mm}
\subsection{Global Convergence for Concave Reward Functionals}\label{sec:ConvergenceAnalysis}
\vspace{-1mm}
If the reward functional $R$ is concave, i.e., for any $\mu, \nu \in \mathcal{P}$ and $\theta \in [0,1]$, we have
$
    R(\theta \mu + (1-\theta) \nu) \geq \theta R(\mu) + (1-\theta) R(\nu),
$
then we can prove the global convergence of the regularized steepest guidance under some structural assumptions.
For each $t \in [\varepsilon,1]$, let $\pi_t^* \in \argmax_{\mu \in \mathcal P}\mathcal V(t,\mu)$.
Following the literature~\citep{nitanda2022convex,chizatmean}, we assume the {\it uniform log-Sobolev inequality} and additional regularity conditions:
\begin{assumption}\label{assump:lsi}
    (Uniform LSI) For any $t \in [\varepsilon, 1]$, there exists a constant $\alpha_t > 0$ such that for any $\mu \in \mathcal{P}$ and smooth function $g: \R^d \to \R$, we have
    \begin{align*}
        \Expec[\nu_t^\mu]{g^2(Y_t) \log g^2(Y_t)} -
        \Expec[\nu_t^\mu]{g^2(Y_t)} \log E_{\nu_t^\mu}[g^2(Y_t)] \leq \frac{2}{\alpha_t} \Expec[\nu_t^\mu]{\|\grad_y g(Y_t)\|^2},
    \end{align*}
    where $\nu_t^\mu(y) \propto \exp(\eta \fdv{V(t, \mu)}{\mu}(y)) \cdot \pi_t(y)$.
    (Regularity condition) Furthermore, defining
    \begin{align*}
        s_t(y) & := \grad_y \log \frac{\pi_t^*(y)}{\pi_t(y)},~\Phi_t(y) := \norm{s_t(y)}^2, ~\Psi_t(y) := \operatorname{div} s_t(y) + s_t(y) \cdot \grad_y \log \pi_t^*(y),
    \end{align*}
    there exists $C_t < \infty$ such that
    $
        \sup_y \Phi_t(y) \leq C_t,~\sup_y \abs{\Psi_t(y)} \leq C_t.
    $
\end{assumption}
The LSI condition can be established through several technical tools \citep{pmlr-v178-chewi22a};
for example, if $\pi_t$ satisfies the LSI and the oscillation of $\eta \fdv{V(t, \mu)}{\mu}(\cdot)$ is bounded, then the Bakry-Emery and Holley-Stroock arguments yield the LSI condition of $\nu_t^\mu$ \citep{10.1007/BFb0075847,holley1987logarithmic}.
Since $\pi_t~(t < 1)$ is smoothed by a Gaussian, we expect the LSI constant $\alpha_t$ to be larger than that of the target distribution.
Under the above assumption, we have the following result.
\begin{theorem}\label{thm:convergence}
    Assume that $R$ is concave and Assumption~\ref{assump:lsi} holds.
    Then, we have
    \begin{align*}
        J_\eta(\pi_1^*)-J_\eta(\pi_1^g)
         & = O\left(\varepsilon + \frac{\eta}{\lambda^2}
        \int_\varepsilon^1\frac{C_t^2}{\alpha_t^2}w_\lambda(t)\d t\right).
    \end{align*}
    where $w_\lambda(t) := \lambda\sigma_t^2\alpha_t / \eta \cdot \exp(-\frac{\lambda}{\eta} \int_t^1 \sigma_s^2\alpha_s \d s)$
    is a weight function that satisfies $\int_\varepsilon^1 w_\lambda(t) \d t \leq 1$.
\end{theorem}\vspace{-2mm}
See Appendix~\ref{proof:convergence} for the proof.
Thus, if $\int_\varepsilon^1 C_t^2 / \alpha_t^2 \cdot w_\lambda(t)\d t$ is finite, by setting $\lambda$ sufficiently large and $\varepsilon$ sufficiently small,
we can ensure that the generated distribution $\pi_1^g$ is close to the optimal distribution $\pi_1^*$.
Note that while the assumptions are similar,
the proof techniques are different from those of mean-field Langevin dynamics~\citep{nitanda2022convex,chizatmean} since the objective functional $\mathcal{V}(t, \mu)$ is time-dependent and
we carefully handle this time-dependency in the proof.
In that sense, this convergence theorem can be viewed as a functional generalization of the well-known convergence analysis for Langevin dynamics \citep{bakry2014analysis}; specifically, it extends to distribution-dependent (mean-field) and time-dependent objective.

\vspace{-2mm}
\section{Numerical Experiments}
\vspace{-2mm}
\begin{figure}[t]
    \centering
    \includegraphics[width=0.9\textwidth]{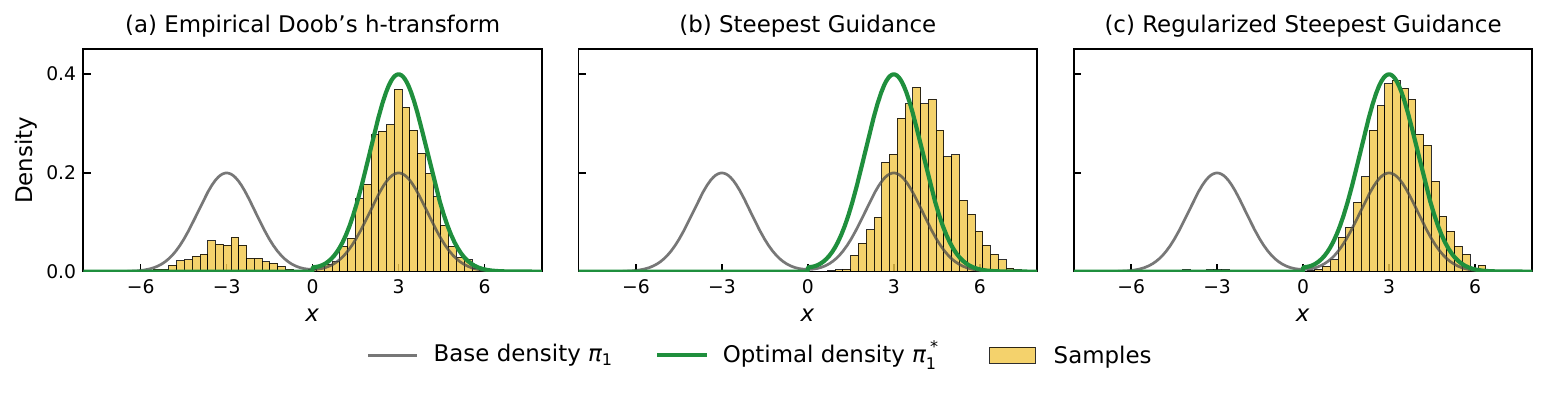}\vspace{-2mm}
    \caption{Toy experiment comparing Doob's $h$-transform (left), Steepest Guidance (middle), and Regularized Steepest Guidance (right).}
    \label{fig:toy_experiment}
    \vspace{-2mm}
\end{figure}

\begin{table*}[t]
    \centering
    \caption{Text-to-image reward summary for a prompt ``A portrait photo of a golden-yellow lion". Results are reported as mean $\pm$ standard deviation over 32 generated images for each model, reward function, and method.}\vspace{-2mm}
    \label{tab:t2i_reward_summary}
    \begin{tabular}{@{}llrrrr@{}}
\toprule
Model & Reward & Unguided & Steepest (Ours) & DOIT & SVDD \\
\midrule
SD & Blueness & $-0.72 \pm 0.06$ & $\mathbf{-0.12} \pm 0.21$ & $-0.69 \pm 0.06$ & $-0.67 \pm 0.06$ \\
 & ImageReward & $-0.39 \pm 0.49$ & $\mathbf{0.86} \pm 0.53$ & $0.10 \pm 0.45$ & $0.68 \pm 0.44$ \\
 & PickScore & $20.97 \pm 0.48$ & $\mathbf{21.90} \pm 0.48$ & $21.38 \pm 0.45$ & $21.84 \pm 0.44$ \\
 & Compressibility & $-14.35 \pm 1.80$ & $\mathbf{-5.25} \pm 1.65$ & $-14.17 \pm 1.74$ & $-13.54 \pm 1.77$ \\
\midrule
FLUX & Blueness & $-0.33 \pm 0.05$ & $\mathbf{-0.13} \pm 0.06$ & $-0.29 \pm 0.04$ & $-0.23 \pm 0.04$ \\
 & ImageReward & $0.05 \pm 0.28$ & $\mathbf{1.66} \pm 0.24$ & $0.45 \pm 0.30$ & $0.92 \pm 0.29$ \\
 & PickScore & $21.45 \pm 0.53$ & $\mathbf{22.62} \pm 0.31$ & $22.10 \pm 0.38$ & $22.42 \pm 0.30$ \\
 & Compressibility & $-8.26 \pm 0.77$ & $\mathbf{-4.68} \pm 0.56$ & $-7.61 \pm 0.67$ & $-6.59 \pm 0.49$ \\
\bottomrule
\end{tabular}
\vspace{-2mm}
\end{table*}

\begin{figure}[t]
    \centering
    \begin{minipage}[t]{0.4\textwidth}
        \centering
        \scalebox{1}[0.9]{\includegraphics[width=\linewidth]{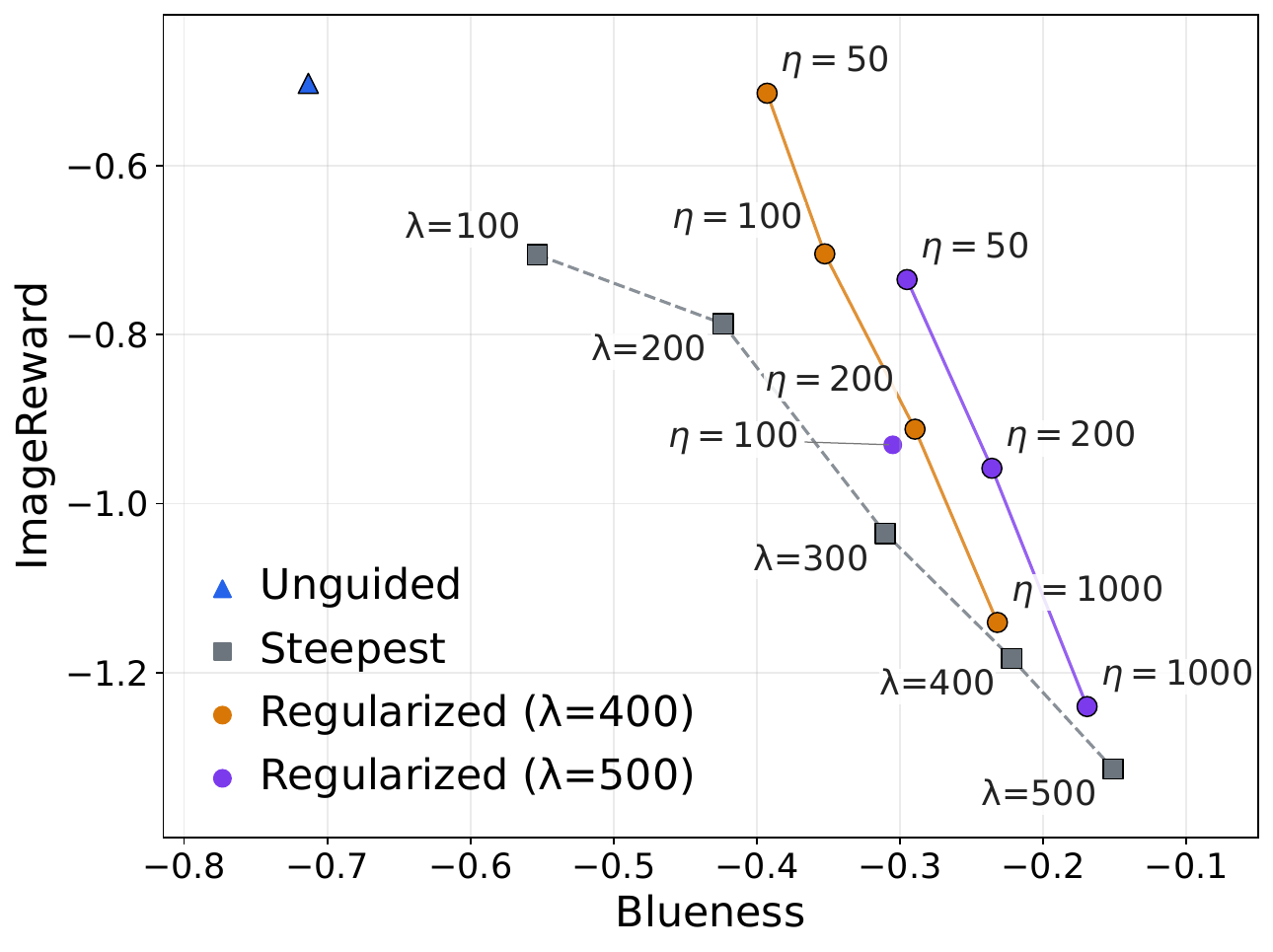}}
        \caption{Trade-off between blueness and image reward for (Regularized) Steepest Guidance.}
        \label{fig:t2i_lion_kl_pareto}
    \end{minipage}
    \hfill
    \begin{minipage}[t]{0.55\textwidth}
        \centering
        \scalebox{1}[0.9]{\includegraphics[width=\linewidth]{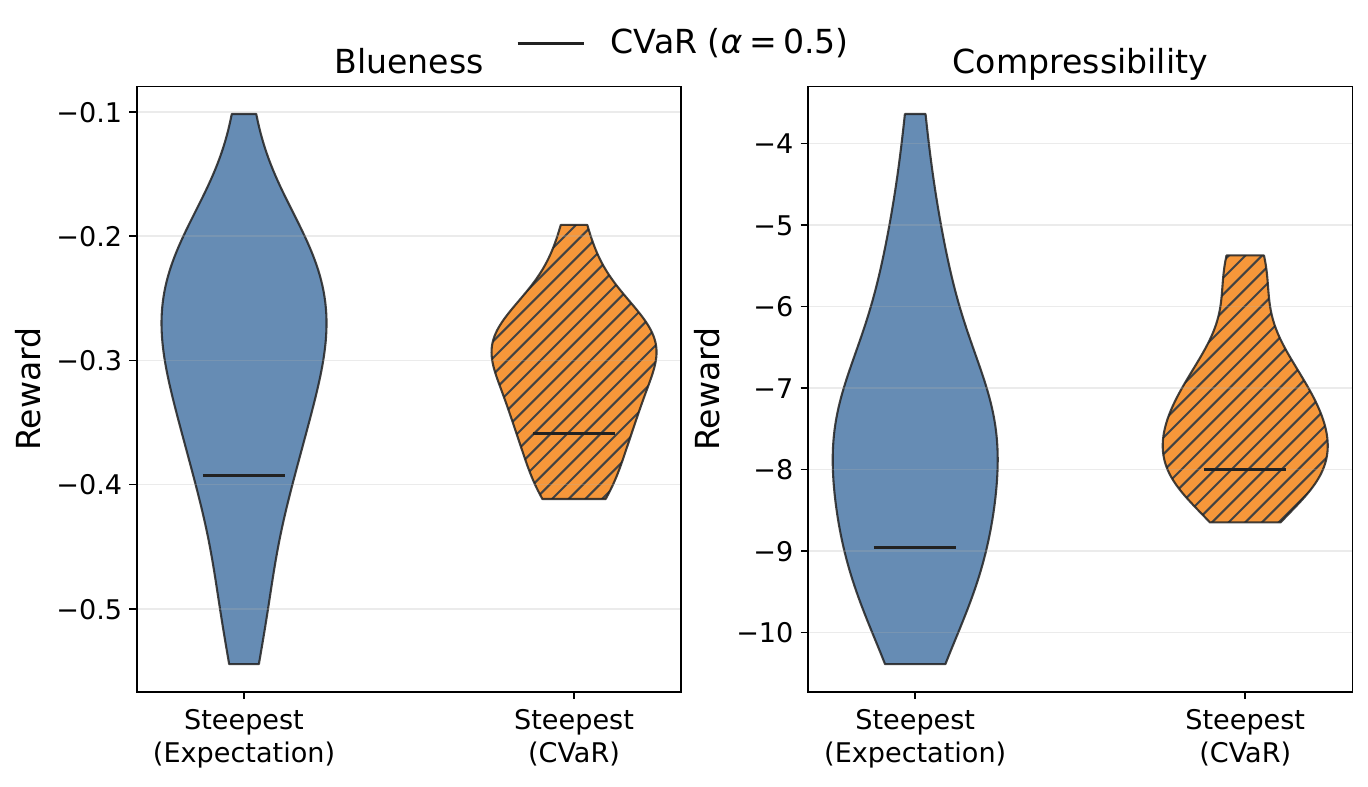}}
        \caption{Comparison between Steepest Guidance with and without CVaR. Horizontal lines indicate the lower-tail CVaR with $\alpha=0.5$.}
        \label{fig:cvar}
    \end{minipage}
    \vspace{-5mm}
\end{figure}

In this section, we evaluate the performance of the proposed steepest guidance method on synthetic and image generation tasks.
\vspace{-1mm}
\subsection{Toy Experiments}
\vspace{-1mm}
First, we consider a toy experiment to illustrate the differences between Steepest Guidance and Doob's $h$-transform with finite lookahead samples.
We consider a simple 1D Gaussian mixture model $\pi_1 = \frac{1}{2} \mathcal{N}(x\mid -3.0, 1.0) + \frac{1}{2} \mathcal{N}(x\mid 3.0, 1.0)$
as the target distribution and a step reward $r(y) = 10 \cdot \mathbf{1}\{y \geq 0\}$.
That is, the objective is to select the right mode in the positive region.

Fig.~\ref{fig:toy_experiment} illustrates the distributions obtained by each method and the optimal distribution derived analytically.
We see that 1) empirical Doob's guidance sometimes fails to select the right mode while ideally it matches the optimal distribution,
2) steepest guidance successfully selects the right mode but the distribution fails to match the optimal distribution,
and 3) regularized steepest guidance successfully approximates the optimal distribution.
The results match our theoretical analysis.


\subsection{Image Generation}
Next, we evaluate the proposed methods on image generation tasks
using Stable Diffusion v1.5 (SD)~\citep{Rombach_2022_CVPR} and FLUX.1 [dev] (FLUX)~\citep{flux2024}.
As reward functions, we consider Blueness (how blue an image is), PickScore~\citep{kirstain2023pick},
ImageReward~\citep{xu2023imagereward}, and Compressibility (scaled negative compression size).
For SD experiments using ImageReward or PickScore, we set $k=8$;
for all other experiments, we use $k=4$.
See Appendix~\ref{sec:experimental_details} for details of the experimental setup
and additional results including results for other prompts, generated images and ablation studies on $k$.

\textbf{Steepest Guidance consistently outperforms baselines:}
Here, we consider expected reward maximization and evaluate the performance of the proposed steepest guidance against
baselines including empirical Doob's transform and SVDD~\citep{li2024derivative}, which is a state-of-the-art selection-based approach.
We tune the hyperparameter $\lambda$ separately for steepest guidance and Doob's $h$-transform for each model and reward,
using random seeds distinct from those used for final evaluation.
Table~\ref{tab:t2i_reward_summary} shows that the steepest guidance outperforms baselines across reward functions and models.

\textbf{Regularized Steepest Guidance improves the reward--quality trade-off:}
We compare the performance of regularized steepest guidance and (vanilla) steepest guidance.
Since the terminal KL divergence is not directly measurable, we use ImageReward as a proxy to assess the extent
to which the guided distribution deviates from the original distribution in terms of prompt alignment.
Fig.~\ref{fig:t2i_lion_kl_pareto} shows that regularized steepest guidance (orange and purple curves) achieves better trade-offs between Blueness and ImageReward than vanilla steepest guidance (gray dotted curve)
by controlling the strength of the regularization parameter $\eta$.

\textbf{Steepest Guidance can handle non-linear rewards:}
To demonstrate the effectiveness of steepest guidance for non-linear reward functionals, we consider CVaR as a reward functional.
We set $\alpha=0.5$. Fig.~\ref{fig:cvar} shows the distribution of the rewards obtained by maximizing the expectation (left) and CVaR (right).
We see that the proposed method achieves better CVaR (horizontal line) compared to maximizing the expectation.

\section{Conclusion}
We proposed \textit{steepest guidance}, a training-free approach to aligning flow and diffusion models.
Different from existing approaches based on Doob's $h$-transform, steepest guidance directly optimizes the reward functional in a local manner.
As a result, our approach simplifies guidance estimation and accommodates general reward functionals.
Theoretically, we established provable improvement guarantees for the ideal dynamics and global convergence under suitable assumptions.
The experiments demonstrated the effectiveness of our proposed method in text-to-image generation tasks.


\subsubsection*{Acknowledgments}
RH was partially supported by JSPS KAKENHI (24K02905) and JST BOOST (JPMJBS2418).
TS was partially supported by JST CREST (JPMJCR2115) and JST ERATO (JPMJER2601).
This research is supported by the National Research Foundation,
Singapore and the Ministry of Digital Development and Information
under the AI Visiting Professorship Programme (award number AIVP-2024-004).
Any opinions, findings and conclusions or recommendations expressed in this material
are those of the author(s) and do not reflect the views of National Research Foundation,
Singapore and the Ministry of Digital Development and Information.



\bibliography{iclr2027_conference}
\bibliographystyle{iclr2027_conference}

\appendix
\section{Related Work}\label{sec:related_work}
In this section, we provide a detailed discussion on related work and compare our method with previous works on reward-guided generation.
\subsection{Systematic Comparison with Previous Works}
Here, we show in Table~\ref{tab:comparison} a systematic comparison of our method with previous works on reward-guided generation
from the perspective of three key properties: training-free, derivative-free, and the ability to handle non-linear reward functionals.
\begin{table}[h]
    \centering
    \caption{Systematic comparison with previous works on reward-guided generation.
        (1) Training-free: The method does not require additional training of models.
        (2) Derivative-free: The method does not require gradients of the reward functional or the generative process.
        (3) Non-linear reward functional: The method can handle non-linear reward functionals.
    }
    \label{tab:comparison}
    \begin{tabular}{ccccc}
        \toprule
        Method                                                      & (1)          & (2)          & (3)          \\
        \midrule
        Steepest Guidance (Ours)                                    & $\checkmark$ & $\checkmark$ & $\checkmark$ \\
        DOIT~\citep{zhu2026training}                                & $\checkmark$ & $\checkmark$ & $\times$     \\
        FDC~\citep{de2025flow}                                      & $\times$     & $\times$     & $\checkmark$ \\
        Adjoint Matching~\citep{domingo2024adjoint}                 & $\times$     & $\times$     & $\times$     \\
        Reinforce Adjoint Matching~\citep{bergmeister2026reinforce} & $\times$     & $\checkmark$ & $\times$     \\
        SALD~\citep{nitanda2026slowly}                              & $\checkmark$ & $\checkmark$ & $\times$     \\
        Greedy Guidance~\citep{blasingame2025greed}                 & $\checkmark$ & $\times$     & $\times$     \\
        Gradient Guidance~\citep{guo2024gradient}                   & $\checkmark$ & $\times$     & $\times$     \\
        \bottomrule
    \end{tabular}
\end{table}

\subsection{Other Related Works}
\paragraph{Diversity-seeking Generation}
A line of work~\citep{corso2023particle,vinograd2026diverse,zilberstein2024repulsive} proposes diversity-seeking generation methods
utilizing repulsive forces between particles.
Such diversity-seeking generation methods can be interpreted as an optimization of non-linear reward functionals
but they are not applicable to general reward functionals.

\paragraph{Optimization of Non-linear Reward Functionals}
The optimization of general functionals over probability measures is studied in various contexts,
including training of two-layer neural networks~\citep{nitanda2022convex,chizatmean}, reinforcement learning~\citep{zhang2020variational}, and
inference-aware training of LLMs~\citep{takakura2026inference}.

\paragraph{Relation to SALD}
SALD~\citep{nitanda2026slowly} employs dynamics similar to ours. Indeed, for linear reward functionals,
greedy guidance can be regarded as a special case of SALD since SALD considers general guidance which satisfies $g_1(y) = \grad r(y)$.
However, our approach is completely different in its design principles.
First of all, SALD is designed for sampling from tilted distributions, which is the optimal solution to KL-regularized expected reward maximization.
Therefore, it cannot be applied to general non-linear reward functionals.
On the other hand, greedy guidance is designed from the perspective of optimization of probability measures and can handle general reward functionals.
Furthermore, SALD's guidance does not generally align with a direction that improves the expected reward.
Consequently, its improvement guarantee relies on sufficiently slow annealing.
In practice, such slowdown can significantly affect the efficiency of generation.

\section{Examples of Reward Functionals}\label{sec:examples_reward_functionals}
In this section, we provide representative examples of (non-linear) reward functionals for various applications.
Table~\ref{tab:reward_functionals} summarizes representative reward functionals.
We mainly follow~\citet{de2025flow} and have added several reward functionals that are not included in their work.

For entropy maximization, we cannot directly compute the first-order variation of the entropy functional
since it requires knowledge of the density $\mu(x)$. Thus, we propose to use the following relation instead:
\begin{align*}
    \ent(\mu) = -\kl(\mu \mid \pi_1) - \Expec[x\sim\mu]{\log \pi_1(x)}.
\end{align*}
Utilizing the above relation, entropy maximization can be interpreted as KL-regularized expected reward maximization.
In general, we cannot compute $\log \pi_1(x)$ but we can use the score function $\grad \log \pi_1(x)$.
Thus, if the sampling process is differentiable, we can estimate the guidance using a plug-in estimator and the chain rule.

\begin{table}[h]
    \centering
    \caption{Representative reward functionals for various applications. }
    \label{tab:reward_functionals}
    \renewcommand{\arraystretch}{1.35}
    \begin{tabular}{cc}
        \toprule
        Application & Functional $R[\mu]$                                                                                     \\
        \midrule
        Expected reward
                    & $\Expec[x\sim\mu]{r(x)}$                                                                                \\
        KL divergence
                    & $D_{\mathrm{KL}}(\mu\|\pi_1)
        :=\int \mu(x)\log\frac{\mu(x)}{\pi_1(x)}\,\d x$                                                                       \\
        Conditional value-at-risk
                    & $\Expec[x\sim\mu]{r(x)\mid r(x)\leq q_\alpha}$                                                          \\
        Variance
                    & $\Expec[x\sim\mu]{(r(x)-\Expec[x\sim\mu]{r(x)})^2}$                                                     \\
        Entropy
                    & $\ent[\mu] := -\Expec[x\sim\mu]{\log\mu(x)}$                                                            \\

        Optimal experimental design
                    & $\mathsf{s}\!\left(\Expec[x\sim\mu]{
        \Phi(x)\Phi(x)^\top-\lambda I}\right)$                                                                                \\
        Log-barrier
                    & $-\beta\log\!\left(\Expec[x\sim\mu]{c(x)}-C\right)$                                                     \\

        Maximum mean discrepancy
                    & $\operatorname{MMD}_k(\mu\|\pi_1) :=\|m_\mu-m_{\pi_1}\|,~m_\mu:=\Expec[x\sim\mu]{k(x,\cdot)}$           \\
        Best- or worst-case reward
                    & $\Expec[x_1,\dots, x_N \sim \mu^N]{\max_{i} r(x_i)},~\Expec[x_1\dots, x_N \sim \mu^N]{\min_{i} r(x_i)}$ \\
        \bottomrule
    \end{tabular}
\end{table}

\subsection{First-order Variation of Reward Functionals}
Here, we provide the first-order variation of CVaR and Rao's quadratic entropy, which are used in our experiments.
\begin{lemma}[First-order variation of CVaR]
    Let $R_{\mathrm{CVaR}}[\mu] = \Expec[\mu]{r(Y) \mid r(Y) \leq q_\alpha}$ be the CVaR functional, where $q_\alpha$ is the $\alpha$-quantile of $r(Y)$.
    Assume that $\mu$ has a strictly positive density on $\R^d$ and that the law of $r(Y)$ under $\mu$ has a continuous density $f_\mu$ in a neighborhood of $q_\alpha$, with $f_\mu(q_\alpha)>0$.
    Then, the first-order variation of $R_{\mathrm{CVaR}}$ is given by
    \begin{align*}
        \fdv{R_{\mathrm{CVaR}}}{\mu}[\mu](y) = \frac{1}{\alpha} \cdot \min\{r(y) - q_\alpha, 0\}.
    \end{align*}
\end{lemma}
\begin{proof}
    For $\nu\in\mathcal P$, let $\mu_\epsilon=(1-\epsilon)\mu+\epsilon\nu$, and denote the reward distribution function under a measure $\rho$ by $F_\rho$.
    Write $q=q_\alpha$ and define $q_\epsilon=\inf\{z:F_{\mu_\epsilon}(z)\geq\alpha\}$.

    First, we prove the differentiability of $q_\epsilon$ with respect to $\epsilon$.
    Since
    \begin{align}
        F_{\mu_\epsilon}(q_\epsilon)=(1-\epsilon)F_\mu(q_\epsilon)+\epsilon F_\nu(q_\epsilon) = \alpha, \label{eq:implicit-q}
    \end{align}
    we have $|F_\mu(q_\epsilon)-\alpha|\leq\epsilon$.
    For sufficiently small $\epsilon$, we have
    \begin{align*}
        \frac{f_\mu(q)}{2}|q_\epsilon-q|
        \leq |F_\mu(q_\epsilon)-F_\mu(q)|\leq\epsilon,
    \end{align*}
    where the last inequality follows from $|F_\mu(q_\epsilon)-\alpha|\leq\epsilon$.
    Since $f_\mu(q)>0$, we obtain $q_\epsilon-q=O(\epsilon)$, which establishes the differentiability of $q_\epsilon$ with respect to $\epsilon$ at $\epsilon=0$.
    By differentiating Eq.~\eqref{eq:implicit-q} with respect to $\epsilon$, we obtain
    \begin{align*}
        f_{\mu_\epsilon}(q_\epsilon) q'_\epsilon = F_\mu(q_\epsilon) - F_\nu(q_\epsilon),
    \end{align*}
    where $q'_\epsilon = \frac{\d q_\epsilon}{\d \epsilon}$.

    Let
    \begin{align*}
        H_\rho(a):=\int r(y)\mathbf{1}\{r(y)\leq a\}\rho(\d y).
    \end{align*}
    Note that $R_{\mathrm{CVaR}}[\rho]=\alpha^{-1}H_\rho(q_\alpha)$.
    Then, we have
    \begin{align*}
        R_{\mathrm{CVaR}}[\mu_\epsilon]-R_{\mathrm{CVaR}}[\mu]
         & = \frac{1}{\alpha}\left(H_{\mu_\epsilon}(q_\epsilon)-H_\mu(q)\right)                                                                                \\
         & = \frac{1}{\alpha}\left(\underbrace{H_{\mu_\epsilon}(q_\epsilon)-H_\mu(q_\epsilon)}_{=:D_1}+\underbrace{H_\mu(q_\epsilon)-H_\mu(q)}_{=:D_2}\right).
    \end{align*}
    For $D_1$, we have
    \begin{align*}
        D_1 = H_{\mu_\epsilon}(q_\epsilon)-H_\mu(q_\epsilon) = \epsilon \int r(y)\mathbf{1}\{r(y)\leq q_\epsilon\}(\nu-\mu)(\d y)
    \end{align*}
    and
    \begin{align*}
        \lim_{\epsilon \to 0}\frac{D_1}{\epsilon} = \int r(y)\mathbf{1}\{r(y)\leq q\}(\nu-\mu)(\d y).
    \end{align*}
    For $D_2$, we have
    \begin{align*}
        D_2 = H_\mu(q_\epsilon)-H_\mu(q) = qf_\mu(q)(q_\epsilon-q) + o(|q_\epsilon-q|),
    \end{align*}
    and
    \begin{align*}
        \lim_{\epsilon \to 0}\frac{D_2}{\epsilon} = qf_\mu(q)q'_0 = q (F_\mu(q)-F_\nu(q)) = \int q \mathbf{1}\{r(y)\leq q\}(\mu-\nu)(\d y).
    \end{align*}
    Combining the above results, we obtain
    \begin{align*}
        \lim_{\epsilon \to 0}\frac{R_{\mathrm{CVaR}}[\mu_\epsilon]-R_{\mathrm{CVaR}}[\mu]}{\epsilon} & = \frac{1}{\alpha}\left(\lim_{\epsilon \to 0}\frac{D_1}{\epsilon} + \lim_{\epsilon \to 0}\frac{D_2}{\epsilon}\right) \\
                                                                                                     & = \frac{1}{\alpha}\int (r(y) - q)\mathbf{1}\{r(y)\leq q\}(\nu-\mu)(\d y).
    \end{align*}
    This implies that the first-order variation of $R_{\mathrm{CVaR}}$ with respect to $\mu$ is
    \begin{align*}
        \fdv{R_{\mathrm{CVaR}}}{\mu}[\mu](y) = \frac{1}{\alpha}\min\{r(y)-q,0\}.
    \end{align*}
\end{proof}

\begin{lemma}[First-order variation of Rao's quadratic entropy]
    Let $R_{\mathrm{Rao}}[\mu] = -\frac{1}{2}\int \int k(y, y') \mu(\d y) \mu(\d y')$ be Rao's quadratic entropy functional,
    where $k: \R^d \times \R^d \to \R$ is a symmetric positive definite kernel.
    Then, the first-order variation of $R_{\mathrm{Rao}}$ is given by
    \begin{align*}
        \fdv{R_{\mathrm{Rao}}}{\mu}[\mu](y) = -\int k(y, y') \mu(\d y').
    \end{align*}
\end{lemma}
\begin{proof}
    For any $\nu \in \mathcal{P}(\R^d)$, let $h = \nu - \mu$. By the symmetry of $k$, we have
    \begin{align*}
        R_{\mathrm{Rao}}[\nu] - R_{\mathrm{Rao}}[\mu]
         & = -\int \int k(y, y') h(\d y) \mu(\d y')
        - \frac{1}{2}\int \int k(y, y') h(\d y) h(\d y').
    \end{align*}
    The first term is linear in $h$, while the second term is quadratic. Therefore,
    \begin{align*}
        \fdv{R_{\mathrm{Rao}}}{\mu}[\mu](y) = -\int k(y, y') \mu(\d y').
    \end{align*}
\end{proof}

\section{Detailed Algorithm of Steepest Guidance}
\begin{algorithm}[H]
    \caption{Steepest Guidance for Reward-Guided Generation}
    \begin{algorithmic}[1]\label{alg:steepest-guidance}
        \STATE Choose a time grid $\varepsilon=t_0<t_1<\cdots<t_L=1$ and set $\Delta t_l=t_{l+1}-t_l$.
        \STATE Sample $M$ particles $\{y_0^{(j)}\}_{j=1}^M$ from $\pi_\varepsilon$ by evolving the unguided base process from its native prior to time $\varepsilon$.
        \FOR{$l = 0, \dots, L-1$}
        \STATE Sample $y_{l,i}^{(j)}$ from $\pi_1(\cdot \mid Y_{t_l}=y_l^{(j)})$ for $j=1,\dots,M$ and $i=1,\dots,k$.
        \STATE Construct the empirical terminal measure $\hat{\mu}_{1,l}=\frac{1}{Mk}\sum_{j=1}^M\sum_{i=1}^k\delta_{y_{l,i}^{(j)}}$.
        \STATE Compute the plug-in approximations $\fdv{R}{\mu}[\hat{\mu}_{1,l}](y_{l,i}^{(j)})$ for all $i$ and $j$, using the abuse of notation introduced in the main text.
        \STATE Estimate the steepest guidance for each $j=1,\dots,M$:
        \begin{align*}
            \hat{g}_{t_l}^{\mathrm{steepest}}(y_l^{(j)})
            = \lambda \sigma_{t_l}^2 \frac{1}{k}\sum_{i=1}^k
            \fdv{R}{\mu}[\hat{\mu}_{1,l}](y_{l,i}^{(j)})
            \grad_{y_l^{(j)}}\log\pi_1(y_{l,i}^{(j)}\mid Y_{t_l}=y_l^{(j)}).
        \end{align*}
        \STATE Update the particles:
        \begin{align*}
            y_{l+1}^{(j)}
            = y_l^{(j)}
            + \left(b_{t_l}(y_l^{(j)})+\hat{g}_{t_l}^{\mathrm{steepest}}(y_l^{(j)})\right)\Delta t_l
            + \sigma_{t_l}\sqrt{\Delta t_l}\,\xi_l^{(j)},
            \quad \xi_l^{(j)}\sim\mathcal N(0,I).
        \end{align*}
        \ENDFOR
        \STATE \textbf{Return:} $\{y_L^{(j)}\}_{j=1}^M$ as the generated samples.
    \end{algorithmic}
\end{algorithm}

\section{Auxiliary Lemmas}
\begin{lemma}[Corollary 18 in~\citep{albergo2025stochastic}]\label{lem:equivalence_sde}
    Let $Y_t$ be the solution of the SDE
    \begin{align*}
        \d Y_t = b_t(Y_t) \d t + \sigma_t \d W_t,
    \end{align*}
    and let $\tilde Y_t$ be the solution of the SDE
    \begin{align*}
        \d \tilde Y_t = (b_t(\tilde Y_t) + g_t(\tilde Y_t)) \d t + \sqrt{1 + 2\lambda} \cdot \sigma_t \d W_t,
    \end{align*}
    where $g_t(y) = \lambda \sigma_t^2 \grad_y \log \pi_t(y)$.
    Then the marginal distribution $\tilde \pi_t$ of $\tilde Y_t$
    coincides with the marginal distribution $\pi_t$ of $Y_t$ for all $t \in [0, 1]$, if $\tilde \pi_0 = \pi_0$.
\end{lemma}

\begin{lemma}[Score Function of Flow Models]\label{lem:score_flow}
    For a flow model, the score function $\grad \log \pi_t(y)$ can be expressed as
    \begin{align*}
        \grad \log \pi_t(y) = \frac{tv_t(y) - y}{1-t}
    \end{align*}
    where $v_t(y)$ is the velocity field of the flow model.
\end{lemma}
\begin{proof}
    From the definition of the flow model, we have $Y_t = t Y_1 + (1 - t) Y_0$, where $Y_0 \sim \mathcal{N}(0, I)$ and $Y_1 \sim \pi_1$.
    Therefore, we have $\frac{Y_t - t Y_1}{1 - t} = Y_0 \sim \mathcal{N}(0, I)$ and
    \begin{align*}
        \pi_t(y) & = \int \pi_t(y \mid Y_1 = z) \pi_1(z) \d z                                         \\
                 & = \int C \cdot \exp\ab(-\frac{1}{2} \norm{\frac{y - t z}{1 - t}}^2) \pi_1(z) \d z,
    \end{align*}
    where $C$ is a normalization constant.
    Differentiating $\pi_t(y)$ with respect to $y$, we have
    \begin{align*}
        \grad \pi_t(y) & = \int C \cdot \exp\ab(-\frac{1}{2} \norm{\frac{y - t z}{1 - t}}^2) \cdot \frac{t z - y}{(1 - t)^2} \pi_1(z) \d z \\
                       & = \int \pi_t(y \mid Y_1 = z) \cdot \frac{t z - y}{(1 - t)^2} \pi_1(z) \d z                                        \\
                       & = \int \pi_1(z \mid Y_t = y) \cdot \frac{t z - y}{(1 - t)^2} \pi_t(y) \d z                                        \\
                       & = \pi_t(y) \cdot \Expec{\frac{tY_1 - Y_t}{(1 - t)^2} \mid Y_t = y}.
    \end{align*}
    Thus, the score function is given by
    \begin{align*}
        \grad \log \pi_t(y) & = \frac{\grad \pi_t(y)}{\pi_t(y)}                                                                         \\
                            & = \Expec{\frac{tY_1 - Y_t}{(1 - t)^2} \mid Y_t = y}                                                       \\
                            & = \Expec{\frac{t\{(1-t)(Y_1 - Y_0) + \overbrace{tY_1 + (1-t)Y_0}^{=Y_t}\} - Y_t}{(1 - t)^2} \mid Y_t = y} \\
                            & = \Expec{\frac{t(1-t)(Y_1 - Y_0) - (1 - t)Y_t}{(1 - t)^2} \mid Y_t = y}                                   \\
                            & = \Expec{\frac{t(Y_1 - Y_0) - Y_t}{(1 - t)} \mid Y_t = y}                                                 \\
                            & = \frac{tv_t(y) - y}{1-t},
    \end{align*}
    where the last equality follows from the definition of the velocity field $v_t(y) = \Expec{Y_1 - Y_0 \mid Y_t = y}$.
\end{proof}

\begin{lemma}[Conditional Distribution of Memoryless Flow]\label{lem:memoryless_flow_conditional}
    Assume that the flow model uses the memoryless noise schedule
    \begin{align*}
        \sigma_t^2 = \frac{2(1-t)}{t}.
    \end{align*}
    Then, for any $t \in [\varepsilon, 1)$,
    \begin{align*}
        Y_t \mid Y_1 = z \sim \mathcal{N}\ab(tz, (1-t)^2 I).
    \end{align*}
\end{lemma}
\begin{proof}
    The time reversal of Eq.~\eqref{eq:flow_sde} is given by
    \begin{align*}
        \d X_t = \ab(-v_{1-t}(X_t) + \frac{\sigma_{1-t}^2}{2} \grad \log \pi_{1-t}(X_t)) \d t + \sigma_{1-t} \d W_t,
        \qquad X_0 \sim \pi_1.
    \end{align*}
    Since $\sigma_{1-t}^2 = 2t/(1-t)$, Lemma~\ref{lem:score_flow} yields
    \begin{align*}
        \d X_t
         & = \ab(-v_{1-t}(X_t) + \frac{t}{1-t}
        \cdot \frac{(1-t)v_{1-t}(X_t) - X_t}{t}) \d t
        + \sqrt{\frac{2t}{1-t}} \d W_t         \\
         & = -\frac{X_t}{1-t} \d t
        + \sqrt{\frac{2t}{1-t}} \d W_t.
    \end{align*}
    To solve this linear SDE, define $Z_s := X_s/(1-s)$.
    Then, we have
    \begin{align*}
        \d Z_s
         & = \frac{1}{1-s}\d X_s + \frac{X_s}{(1-s)^2}\d s                                \\
         & = \frac{1}{1-s}\left(-\frac{X_s}{1-s}\d s + \sqrt{\frac{2s}{1-s}}\d W_s\right)
        + \frac{X_s}{(1-s)^2}\d s                                                         \\
         & = \sqrt{\frac{2s}{(1-s)^3}}\d W_s.
    \end{align*}
    Conditional on $X_0=z$, we have $Z_0=z$. Integrating the above equation from $0$ to $s$ and multiplying by $1-s$ gives
    \begin{align*}
        X_s = (1-s)z
        + (1-s)\int_0^s \sqrt{\frac{2r}{(1-r)^3}} \d W_r.
    \end{align*}
    The stochastic integral is centered Gaussian with covariance
    \begin{align*}
        (1-s)^2 \int_0^s \frac{2r}{(1-r)^3}\d r \cdot I
        = s^2 I.
    \end{align*}
    Hence,
    \begin{align*}
        X_s \mid X_0=z \sim \mathcal{N}\ab((1-s)z,s^2I).
    \end{align*}
    Under the time-reversal coupling, $X_s=Y_{1-s}$ and $X_0=Y_1$.
    Taking $s=1-t$ proves
    \begin{align*}
        Y_t \mid Y_1=z \sim \mathcal{N}\ab(tz,(1-t)^2I).
    \end{align*}
\end{proof}

\begin{lemma}[Kolmogorov Backward Equation]\label{lem:kolmogorov_backward}
    Let $K_t$ be the transition kernel of the base SDE from time $t$ to time $1$, and define
    $V(t, \mu) := R[K_t\mu]$. Then, for any
    $t \in [\varepsilon, 1)$ and $\mu \in \mathcal{P}(\mathcal{Y})$,
    \begin{align*}
        \partial_t V(t, \mu) + \mathfrak{L}_t V(t, \mu) = 0,
    \end{align*}
    where
    \begin{align*}
        \mathfrak{L}_t V(t, \mu) & := \Expec[\mu]{\mathcal{L}_t \fdv{V(t, \mu)}{\mu}(y)},         \\
        \mathcal{L}_t f(y)       & := b_t(y) \cdot \grad f(y) + \frac{\sigma_t^2}{2} \Delta f(y).
    \end{align*}
\end{lemma}
\begin{proof}
    The adjoint of $\mathcal{L}_t$ is defined as
    \begin{align*}
        \mathcal{L}_t^*\rho(y) := -\operatorname{div}(b_t(y)\rho(y)) + \frac{\sigma_t^2}{2}\Delta\rho(y).
    \end{align*}
    Fix $t \in [\varepsilon, 1)$ and $\mu \in \mathcal{P}(\mathcal{Y})$, and let $(\mu_s)_{s \in [t, 1]}$ be the marginals of the base SDE with $\mu_t = \mu$.
    By the Chapman--Kolmogorov property, $V(t, \mu) = R[\mu_1] = V(s, \mu_s)$ for every $s \in [t, 1]$.
    Differentiating this identity with respect to $s$ at $s=t$, and using the Fokker--Planck equation $\partial_s \mu_s = \mathcal{L}_s^*\mu_s$, gives
    \begin{align*}
        0
         & = \partial_t V(t, \mu) + \int \fdv{V(t, \mu)}{\mu}(y) \mathcal{L}_t^*\mu(\d y) \\
         & = \partial_t V(t, \mu) + \Expec[\mu]{\mathcal{L}_t \fdv{V(t, \mu)}{\mu}(y)}    \\
         & = \partial_t V(t, \mu) + \mathfrak{L}_t V(t, \mu),
    \end{align*}
    where the second equality follows from integration by parts. This proves the claim.
\end{proof}

\begin{lemma}[Relative Entropy Dissipation Formula]\label{lem:relative_entropy_dissipation}
    Let $\mu_\tau$ and $\pi_\tau$ be the marginal distributions of $Y_\tau$ at time $\tau \in [\varepsilon, 1]$.
    Assume that $\mu_t$ and $\pi_t$ for $t \in [\tau, 1]$ are defined by the marginal distributions of the SDE:
    \begin{align*}
        \d Y_t = b_t(Y_t) \d t + \sigma_t \d W_t,
    \end{align*}
    with initial distributions $\mu_\tau$ and $\pi_\tau$, respectively.
    Then, for almost every $t \in (\tau, 1)$, we have
    \begin{align*}
        \odv{\kl(\mu_t \mid \pi_t)}{t} = -\frac{\sigma_t^2}{2} \Expec[\mu_t]{\left\|\grad_y \log \frac{\d\mu_t}{\d\pi_t}(Y_t)\right\|^2}.
    \end{align*}
\end{lemma}
\begin{proof}
    Let $f_t := \d\mu_t/\d\pi_t$. Both $\mu_t$ and $\pi_t$ satisfy the Fokker--Planck equation
    $\partial_t \rho_t = \mathcal{L}_t^*\rho_t$.
    Hence, differentiating the relative entropy and using integration by parts gives
    \begin{align*}
        \odv{\kl(\mu_t \mid \pi_t)}{t}
         & = \int \log f_t(y) \mathcal{L}_t^*\mu_t(\d y)
        - \int f_t(y) \mathcal{L}_t^*\pi_t(\d y)         \\
         & = \Expec[\mu_t]{\mathcal{L}_t \log f_t(Y_t)}
        - \Expec[\pi_t]{\mathcal{L}_t f_t(Y_t)}.
    \end{align*}
    Since
    \begin{align*}
        \mathcal{L}_t \log f_t
        = \frac{\mathcal{L}_t f_t}{f_t} - \frac{\sigma_t^2}{2}\norm{\grad_y \log f_t}^2,
    \end{align*}
    the preceding display becomes
    \begin{align*}
        \odv{\kl(\mu_t \mid \pi_t)}{t}
         & = -\frac{\sigma_t^2}{2}\Expec[\mu_t]{\norm{\grad_y \log f_t(Y_t)}^2},
    \end{align*}
    which is the desired result.
\end{proof}

\begin{lemma}[Optimality Gap Bounds]\label{lem:regularized_functional_gap}
    Fix $t \in [\varepsilon, 1]$, and assume that $V(t, \cdot)$ is concave.
    Then, for every $\mu \in \mathcal{P}$, $\pi_t^* \in \argmax_{\mu \in \mathcal{P}}\mathcal{V}(t, \mu)$ satisfies
    \begin{align*}
        \kl(\mu \mid \pi_t^*)
        \leq \eta\left(\mathcal{V}(t, \pi_t^*) - \mathcal{V}(t, \mu)\right)
        \leq \kl(\mu \mid \nu_t^\mu).
    \end{align*}
\end{lemma}
\begin{proof}
    See Proposition~1 of~\citet{nitanda2022convex}.
\end{proof}

\begin{lemma}[Initial optimality gap at the cutoff]\label{lem:initial_gap_cutoff}
    Let
    \begin{align*}
        \phi(y) & := \fdv{R}{\mu}[\pi_1](y),                      \\
        L_\phi  & := \operatorname*{ess\,sup}_{\pi_1}\phi
        - \operatorname*{ess\,inf}_{\pi_1}\phi,                   \\
        S_2     & := \Expec[\pi_1]{\|Y_1-\Expec[\pi_1]{Y_1}\|^2}.
    \end{align*}
    Under the bounded first-order-variation assumption, $L_\phi\leq2C_V$.
    Suppose that $R$ is concave, $L_\phi<\infty$, and $S_2<\infty$.
    Let $\pi_\varepsilon^*\in\argmax_{\mu\in\mathcal P}\mathcal V(\varepsilon,\mu)$ and define
    \begin{align*}
        \Delta_\varepsilon
        := \mathcal V(\varepsilon,\pi_\varepsilon^*)
        - \mathcal V(\varepsilon,\pi_\varepsilon).
    \end{align*}
    Then, for the linear flow interpolation with independent endpoints,
    \begin{align*}
        \Delta_\varepsilon
        \leq \frac{S_2}{4\eta}
        \left(e^{\eta L_\phi}-1-\eta L_\phi\right)
        \frac{\varepsilon^2}{(1-\varepsilon)^2}
        = O(\varepsilon^2).
    \end{align*}
    For the diffusion interpolation,
    \begin{align*}
        \Delta_\varepsilon
        \leq \frac{S_2}{4\eta}
        \left(e^{\eta L_\phi}-1-\eta L_\phi\right)
        \frac{\varepsilon}{1-\varepsilon}
        = O(\varepsilon).
    \end{align*}
\end{lemma}
\begin{proof}
    Write
    \begin{align*}
        Y_\varepsilon = a_\varepsilon Y_1+b_\varepsilon Z,
        \qquad Z\sim\mathcal N(0,I),\quad Z\perp Y_1,
    \end{align*}
    where $(a_\varepsilon,b_\varepsilon)=(\varepsilon,1-\varepsilon)$
    for the linear flow interpolation and
    $(a_\varepsilon,b_\varepsilon)=(\sqrt\varepsilon,\sqrt{1-\varepsilon})$
    for the diffusion interpolation. Since $K_\varepsilon\pi_\varepsilon=\pi_1$, the chain rule gives
    \begin{align*}
        f_\varepsilon(y)
        := \fdv{V(\varepsilon,\pi_\varepsilon)}{\mu}(y)
        = \Expec{\phi(Y_1)\mid Y_\varepsilon=y}.
    \end{align*}
    Since $K_\varepsilon$ is linear, concavity of $R$ implies concavity of $V(\varepsilon,\cdot)$.
    Lemma~\ref{lem:regularized_functional_gap} with $\mu=\pi_\varepsilon$ and the definition of $\nu_\varepsilon^{\pi_\varepsilon}$ imply
    \begin{align}
        \Delta_\varepsilon
         & \leq \frac{1}{\eta}\kl(\pi_\varepsilon\mid\nu_\varepsilon^{\pi_\varepsilon}) \notag \\
         & = \frac{1}{\eta}\log\Expec[\pi_\varepsilon]{
            \exp\left(\eta\left\{f_\varepsilon(Y_\varepsilon)
            -\Expec[\pi_\varepsilon]{f_\varepsilon(Y_\varepsilon)}\right\}\right)}.
        \label{eq:initial_gap_mgf}
    \end{align}

    Let $P_y$ denote the conditional law of $Y_1$ given $Y_\varepsilon=y$.
    Since $\phi$ has essential range of length $L_\phi$, Pinsker's inequality gives
    \begin{align*}
        \left|f_\varepsilon(y)-\Expec[\pi_1]{\phi(Y_1)}\right|^2
         & \leq L_\phi^2\operatorname{TV}(P_y,\pi_1)^2 \\
         & \leq \frac{L_\phi^2}{2}\kl(P_y\mid\pi_1).
    \end{align*}
    Averaging over $Y_\varepsilon$ yields
    \begin{align}
        \operatorname{Var}\left(f_\varepsilon(Y_\varepsilon)\right)
        \leq \frac{L_\phi^2}{2}\operatorname{MI}(Y_1;Y_\varepsilon).
        \label{eq:conditional_variance_mi}
    \end{align}
    Let $Q_\varepsilon:=\mathcal N(a_\varepsilon\Expec[\pi_1]{Y_1},b_\varepsilon^2I)$. Then,
    \begin{align*}
        \operatorname{MI}(Y_1;Y_\varepsilon)
         & \leq \Expec[\pi_1]{
        \kl\left(\mathcal N(a_\varepsilon Y_1,b_\varepsilon^2I)\mid Q_\varepsilon\right)} \\
         & = \frac{a_\varepsilon^2}{2b_\varepsilon^2}S_2.
    \end{align*}
    Combining this with~\eqref{eq:conditional_variance_mi} gives
    \begin{align}
        \operatorname{Var}\left(f_\varepsilon(Y_\varepsilon)\right)
        \leq \frac{L_\phi^2S_2}{4}\frac{a_\varepsilon^2}{b_\varepsilon^2}.
        \label{eq:conditional_variance_bound}
    \end{align}

    Set $U_\varepsilon:=f_\varepsilon(Y_\varepsilon)
        -\Expec[\pi_\varepsilon]{f_\varepsilon(Y_\varepsilon)}$.
    Then $\Expec{U_\varepsilon}=0$ and $U_\varepsilon\leq L_\phi$.
    If $L_\phi=0$, then $U_\varepsilon=0$ almost surely and~\eqref{eq:initial_gap_mgf} gives $\Delta_\varepsilon=0$.
    Suppose that $L_\phi>0$. For any centered random variable $U\leq L$, Bennett's moment-generating-function inequality gives
    \begin{align*}
        \log\Expec{e^{\eta U}}
        \leq \frac{\operatorname{Var}(U)}{L^2}
        \left(e^{\eta L}-1-\eta L\right).
    \end{align*}
    Applying this inequality to~\eqref{eq:initial_gap_mgf} and using~\eqref{eq:conditional_variance_bound} gives
    \begin{align*}
        \Delta_\varepsilon
        \leq \frac{S_2}{4\eta}
        \left(e^{\eta L_\phi}-1-\eta L_\phi\right)
        \frac{a_\varepsilon^2}{b_\varepsilon^2}.
    \end{align*}
    Substituting the two choices of $(a_\varepsilon,b_\varepsilon)$ proves the claim.
\end{proof}

\begin{lemma}[Fisher Information Comparison via Pinsker's Inequality]\label{lem:fisher_information_comparison_pinsker}
    Define
    \begin{align*}
        s_t(y)    & := \grad_y \log \frac{\pi_t^*(y)}{\pi_t(y)},                         \\
        \Phi_t(y) & := \norm{s_t(y)}^2,                                                  \\
        \Psi_t(y) & := \operatorname{div} s_t(y) + s_t(y) \cdot \grad_y \log \pi_t^*(y),
    \end{align*}
    and suppose that there exists $C_t < \infty$ such that
    \begin{align*}
        \sup_y \norm{s_t(y)}^2 \leq C_t,
        \qquad
        \sup_y \abs{\Psi_t(y)} \leq C_t.
    \end{align*}
    Then, for every $\mu \in \mathcal{P}$,
    \begin{align*}
        I(\pi_t^* \mid \pi_t) - I(\mu \mid \pi_t)
        \leq {} & \frac{5C_t}{\sqrt{2}}\sqrt{\kl(\mu \mid \pi_t^*)}  \\
                & - I(\mu \mid \pi_t^*)                              \\
        \leq {} & \frac{5C_t}{\sqrt{2}}\sqrt{\kl(\mu \mid \pi_t^*)}.
    \end{align*}
\end{lemma}
\begin{proof}
    We have
    \begin{align*}
        \grad_y \log \frac{\mu(y)}{\pi_t(y)} = \grad_y \log \frac{\mu(y)}{\pi_t^*(y)} + s_t(y).
    \end{align*}
    Expanding the square yields
    \begin{align*}
        I(\mu \mid \pi_t)
        ={} & I(\mu \mid \pi_t^*)
        + 2 \int \mu(y) \grad_y \log \frac{\mu(y)}{\pi_t^*(y)} \cdot s_t(y) \, \d y
        + \Expec[\mu]{\Phi_t}.
    \end{align*}
    By integration by parts,
    \begin{align*}
        \int \mu(y) \grad_y \log \frac{\mu(y)}{\pi_t^*(y)} \cdot s_t(y) \, \d y = -\Expec[\mu]{\Psi_t}.
    \end{align*}
    Furthermore,
    \begin{align*}
        \Expec[\pi_t^*]{\Psi_t} = 0,
        \qquad
        I(\pi_t^* \mid \pi_t) = \Expec[\pi_t^*]{\Phi_t}.
    \end{align*}
    Consequently,
    \begin{align*}
        I(\pi_t^* \mid \pi_t) - I(\mu \mid \pi_t)
        ={} & \left(\Expec[\pi_t^*]{\Phi_t} - \Expec[\mu]{\Phi_t}\right)    \\
            & + 2\left(\Expec[\mu]{\Psi_t} - \Expec[\pi_t^*]{\Psi_t}\right)
        - I(\mu \mid \pi_t^*).
    \end{align*}
    Since $0 \leq \Phi_t \leq C_t$ and $\abs{\Psi_t} \leq C_t$, we have
    \begin{align*}
        \Expec[\pi_t^*]{\Phi_t} - \Expec[\mu]{\Phi_t}
         & \leq C_t\operatorname{TV}(\mu,\pi_t^*),  \\
        \Expec[\mu]{\Psi_t} - \Expec[\pi_t^*]{\Psi_t}
         & \leq 2C_t\operatorname{TV}(\mu,\pi_t^*).
    \end{align*}
    Pinsker's inequality gives
    \begin{align*}
        \operatorname{TV}(\mu,\pi_t^*) \leq \sqrt{\frac{1}{2}\kl(\mu \mid \pi_t^*)}.
    \end{align*}
    Combining these inequalities gives the first inequality in the statement. The second follows from $I(\mu \mid \pi_t^*) \geq 0$.
\end{proof}

\section{Proof of Lemma~\ref{lem:conditional_score}}\label{proof:conditional_score}
By Bayes' rule, we have
\begin{align*}
    \log \pi_1(z \mid Y_t = y) & = \log \pi_t(y \mid Y_1 = z) + \log \pi_1(z) - \log \pi_t(y),
\end{align*}
and differentiating both sides with respect to $y$ gives
\begin{align*}
    \grad_y \log \pi_1(z \mid Y_t = y) & = \grad_y \log \pi_t(y \mid Y_1 = z) - \grad_y \log \pi_t(y).
\end{align*}
The second term $\grad \log \pi_t(y)$ is the score function of $\pi_t$, which is estimated by a neural network for diffusion models
and can be computed from the velocity field for flow models using Lemma~\ref{lem:score_flow}.
Therefore, we only need to compute the first term $\grad \log \pi_t(y \mid Y_1 = z)$.

\paragraph{Flow Models.}
Under the memoryless noise schedule, Lemma~\ref{lem:memoryless_flow_conditional} gives
\begin{align*}
    \pi_t(y \mid Y_1 = z) = C \cdot \exp\ab(-\frac{1}{2(1-t)^2} \norm{y - tz}^2),
\end{align*}
where $C$ is a normalization constant. Thus,
\begin{align*}
    \grad_y \log \pi_t(y \mid Y_1 = z) = -\frac{1}{(1-t)^2} (y - tz).
\end{align*}

\paragraph{Diffusion Models.}
For $t \in [\varepsilon,1)$, let $s=1-t \in (0,1-\varepsilon]$. The solution of the forward SDE satisfies
\begin{align*}
    X_s \mid X_0 = z \sim \mathcal{N}(\sqrt{1-s}z, sI).
\end{align*}
The reverse process is the time reversal of the forward process, so $Y_t = X_{1-t}$ and $Y_1 = X_0$ under the time-reversal coupling. Therefore,
\begin{align*}
    Y_t \mid Y_1 = z \sim \mathcal{N}(\sqrt{t}z, (1-t)I)
\end{align*}
and we have
\begin{align*}
    \grad_y \log \pi_t(y \mid Y_1 = z) = -\frac{1}{1-t}(y - \sqrt{t}z).
\end{align*}

\section{Proof of Proposition~\ref{prop:reward_improvement}}\label{proof:reward_improvement}
The marginal distribution $\pi_t^g$ of the guided SDE satisfies the Fokker--Planck equation
\begin{align*}
    \partial_t \pi_t^g = \mathcal{L}_t^*\pi_t^g - \operatorname{div}(g_t\pi_t^g).
\end{align*}
Therefore, the chain rule for functions of measures and integration by parts give
\begin{align*}
    \odv{V(t, \pi_t^g)}{t}
     & = \partial_t V(t, \pi_t^g) + \int \fdv{V(t, \pi_t^g)}{\mu}(y)\partial_t\pi_t^g(\d y)    \\
     & = \partial_t V(t, \pi_t^g) + \Expec[\pi_t^g]{\mathcal{L}_t \fdv{V(t, \pi_t^g)}{\mu}(y)}
    + \Expec[\pi_t^g]{g_t(y) \cdot \grad \fdv{V(t, \pi_t^g)}{\mu}(y)}                          \\
     & = \partial_t V(t, \pi_t^g) + \mathfrak{L}_t^g V(t, \pi_t^g),
\end{align*}
where
\begin{align*}
    \mathfrak{L}_t^g V(t, \mu) & := \mathfrak{L}_t V(t, \mu) + \Expec[\mu]{g_t(y) \cdot \grad \fdv{V(t, \mu)}{\mu}(y)}.
\end{align*}

Applying Lemma~\ref{lem:kolmogorov_backward}, we have
\begin{align*}
    \odv{V(t, \pi_t^g)}{t} & = \partial_t V(t, \pi_t^g) + \mathfrak{L}_t V(t, \pi_t^g) + \Expec[\pi_t^g]{g_t(y) \cdot \grad \fdv{V(t, \pi_t^g)}{\mu}(y)} \\
                           & = \Expec[\pi_t^g]{g_t(y) \cdot \grad \fdv{V(t, \pi_t^g)}{\mu}(y)}.
\end{align*}
Integrating both sides from $t = \varepsilon$ to $t = 1$ and using $K_\varepsilon\pi_\varepsilon=\pi_1$, we have
\begin{align*}
    V(1, \pi_1^g) - V(\varepsilon, \pi_\varepsilon)
     & = \int_\varepsilon^1 \Expec[\pi_t^g]{g_t(y) \cdot \grad \fdv{V(t, \pi_t^g)}{\mu}(y)} \d t \\
     & = R[\pi_1^g] - R[\pi_1].
\end{align*}
This completes the proof of the reward improvement.

Since the base and guided processes share the initial marginal $\pi_\varepsilon$, the KL identity follows from Girsanov's theorem and the admissibility conditions:
\begin{align*}
    \kl(\mathbb P_{\pi^g}\mid\mathbb P_\pi)
    = \frac{1}{2}\int_\varepsilon^1 \Expec[\pi_t^g]{\frac{\|g_t(Y_t)\|^2}{\sigma_t^2}}\d t.
\end{align*}

\section{Proof of Theorem~\ref{thm:improvement_steepest}}\label{proof:improvement_steepest}
Applying Proposition~\ref{prop:reward_improvement} to the steepest guidance $g_t(y) = \lambda \sigma_t^2 \grad_y \fdv{V(t, \pi_t^g)}{\mu}(y)$, we have
\begin{align*}
    R[\pi_1^g] - R[\pi_1]                         & = \lambda \int_\varepsilon^1 \Expec[\pi_t^g]{\sigma_t^2 \left\|\grad_y \fdv{V(t, \pi_t^g)}{\mu}(y)\right\|^2} \d t,             \\
    \kl(\mathbb{P}_{\pi^g} \mid \mathbb{P}_{\pi}) & = \frac{\lambda^2}{2} \int_\varepsilon^1 \Expec[\pi_t^g]{\sigma_t^2 \left\|\grad_y \fdv{V(t, \pi_t^g)}{\mu}(y)\right\|^2} \d t.
\end{align*}
Combining the above two equations, we have
\begin{align*}
    R[\pi_1^g] - R[\pi_1] - \frac{1}{\lambda} \kl(\mathbb{P}_{\pi^g} \mid \mathbb{P}_{\pi}) = \frac{\lambda}{2} \int_\varepsilon^1 \Expec[\pi_t^g]{\sigma_t^2 \left\|\grad_y \fdv{V(t, \pi_t^g)}{\mu}(y)\right\|^2} \d t
\end{align*}
This completes the proof.

\section{Proof of Corollary~\ref{cor:improvement_steepest_kl}}\label{proof:improvement_steepest_kl}
Define
\begin{align*}
    A := \int_\varepsilon^1 \Expec[\pi_t^g]{\sigma_t^2 \left\|\grad_y \fdv{V(t, \pi_t^g)}{\mu}(Y_t)\right\|^2} \d t.
\end{align*}
From Theorem~\ref{thm:improvement_steepest}, we have
\begin{align*}
    R[\pi_1^g] - R[\pi_1] = \lambda A,
    \qquad
    \kl(\mathbb{P}_{\pi^g} \mid \mathbb{P}_{\pi}) = \frac{\lambda^2}{2} A.
\end{align*}
From the information processing inequality, we have $\kl(\pi^g_1 \mid \pi_1) \leq \kl(\mathbb{P}_{\pi^g} \mid \mathbb{P}_{\pi})$.
Therefore,
\begin{align*}
    J_\eta(\pi_1^g) - J_\eta(\pi_1)
     & = R[\pi_1^g] - \frac{1}{\eta}\kl(\pi^g_1 \mid \pi_1) - R[\pi_1] \\
     & \geq \lambda A - \frac{\lambda^2}{2\eta} A                      \\
     & = \lambda \left(1 - \frac{\lambda}{2\eta}\right) A              \\
     & \geq 0,
\end{align*}
where the last inequality follows from $0 < \lambda \leq 2\eta$.
This completes the proof.


\section{Proof of Proposition~\ref{prop:equivalence_sde}}\label{proof:equivalence_sde}
The result follows from Lemma~\ref{lem:equivalence_sde}.

\section{Proof of Theorem~\ref{thm:improvement_steepest_terminal_kl}}\label{proof:improvement_steepest_terminal_kl}
By Proposition~\ref{prop:equivalence_sde}, it suffices to analyze the guided SDE with $g_t(y) = \lambda \sigma_t^2 \grad_y \fdv{\mathcal{V}(t, \pi_t^g)}{\mu}(y)$.
Since $V(t, \mu) = R[K_t\mu]$ satisfies the Kolmogorov backward equation, Lemmas~\ref{lem:kolmogorov_backward} and~\ref{lem:relative_entropy_dissipation} give
\begin{align*}
    \partial_t \mathcal{V}(t, \mu) + \mathfrak{L}_t \mathcal{V}(t, \mu)
    = \frac{\sigma_t^2}{2\eta} I(\mu \mid \pi_t) \geq 0.
\end{align*}
The chain-rule calculation for the guided SDE, together with
$g_t(y) = \lambda \sigma_t^2 \grad_y \fdv{\mathcal{V}(t, \pi_t^g)}{\mu}(y)$, then yields
\begin{align*}
    \odv{\mathcal{V}(t, \pi_t^g)}{t}
     & = \partial_t \mathcal{V}(t, \pi_t^g) + \mathfrak{L}_t \mathcal{V}(t, \pi_t^g)
    + \Expec[\pi_t^g]{g_t(Y_t) \cdot \grad_y \fdv{\mathcal{V}(t, \pi_t^g)}{\mu}(Y_t)}                             \\
     & = \frac{\sigma_t^2}{2\eta} I(\pi_t^g \mid \pi_t)
    + \lambda \Expec[\pi_t^g]{\sigma_t^2 \left\|\grad_y \fdv{\mathcal{V}(t, \pi_t^g)}{\mu}(Y_t)\right\|^2}        \\
     & \geq \lambda \Expec[\pi_t^g]{\sigma_t^2 \left\|\grad_y \fdv{\mathcal{V}(t, \pi_t^g)}{\mu}(Y_t)\right\|^2}.
\end{align*}
Integrating over $t \in [\varepsilon, 1]$ and using
\begin{align*}
    \mathcal{V}(\varepsilon, \pi_\varepsilon)
    = R[K_\varepsilon\pi_\varepsilon]
    - \eta^{-1}\kl(\pi_\varepsilon \mid \pi_\varepsilon)
    = R[\pi_1] = J_\eta(\pi_1)
\end{align*}
and $\mathcal{V}(1, \pi_1^g) = J_\eta(\pi_1^g)$ proves the claim.

\section{Proof of Theorem~\ref{thm:convergence}}\label{proof:convergence}
For $t \in [\varepsilon,1]$, let
$\Delta_t = \mathcal{V}(t, \pi_t^*) - \mathcal{V}(t, \pi_t^g)$.
Since $\pi_\varepsilon^g=\pi_\varepsilon$, the initial gap is
\begin{align*}
    \Delta_\varepsilon
    = \mathcal V(\varepsilon,\pi_\varepsilon^*)
    - \mathcal V(\varepsilon,\pi_\varepsilon) \geq 0.
\end{align*}
Since we assume that $\pi_1$ has a finite second moment, and the first-order variation of $R$ is bounded,
it follows from Lemma~\ref{lem:initial_gap_cutoff} that $\Delta_\varepsilon=O(\varepsilon^2)$ for the linear flow interpolation and $\Delta_\varepsilon=O(\varepsilon)$ for the diffusion interpolation.
For simplicity, write $\nu_t^g := \nu_t^{\pi_t^g}$. By the definition of $\mathcal{V}$,
\begin{align*}
    \grad_y \fdv{\mathcal{V}(t, \mu)}{\mu}(y)
    = -\frac{1}{\eta}\grad_y \log \frac{\mu(y)}{\nu_t^\mu(y)}.
\end{align*}
From the first-order optimality condition, $\pi_t^* = \nu_t^{\pi_t^*}$.
Hence, $\fdv{\mathcal{V}(t, \pi_t^*)}{\mu}$ is constant in $y$ and
$\mathfrak{L}_t\mathcal{V}(t, \pi_t^*) = 0$.
Moreover, Lemmas~\ref{lem:kolmogorov_backward} and~\ref{lem:relative_entropy_dissipation} imply
\begin{align*}
    \partial_t\mathcal{V}(t, \mu) + \mathfrak{L}_t\mathcal{V}(t, \mu)
    = \frac{\sigma_t^2}{2\eta}I(\mu \mid \pi_t).
\end{align*}
In particular, since $\mathfrak{L}_t\mathcal{V}(t, \pi_t^*) = 0$, we have
\begin{align*}
    \partial_t\mathcal{V}(t, \pi_t^*)
    = \frac{\sigma_t^2}{2\eta}I(\pi_t^* \mid \pi_t),
\end{align*}
whereas
\begin{align*}
    \partial_t\mathcal{V}(t, \pi_t^g) + \mathfrak{L}_t\mathcal{V}(t, \pi_t^g)
    = \frac{\sigma_t^2}{2\eta}I(\pi_t^g \mid \pi_t).
\end{align*}
By Proposition~\ref{prop:equivalence_sde}, we may use the Fokker--Planck equation of the marginally equivalent ideal SDE:
\begin{align*}
    \partial_t\pi_t^g
    = \mathcal{L}_t^*\pi_t^g
    - \lambda\sigma_t^2\operatorname{div}\left(\pi_t^g\grad_y\fdv{\mathcal{V}(t, \pi_t^g)}{\mu}\right).
\end{align*}
The envelope theorem and integration by parts now yield
\begin{align*}
    \frac{\d}{\d t}\Delta_t
    ={} & \partial_t\mathcal{V}(t, \pi_t^*)
    - \left(\partial_t\mathcal{V}(t, \pi_t^g) + \mathfrak{L}_t\mathcal{V}(t, \pi_t^g)\right)                  \\
        & - \lambda\sigma_t^2\Expec[\pi_t^g]{\left\|\grad_y\fdv{\mathcal{V}(t, \pi_t^g)}{\mu}(Y_t)\right\|^2} \\
    ={} & -\frac{\lambda\sigma_t^2}{\eta^2}I(\pi_t^g \mid \nu_t^g)
    + \frac{\sigma_t^2}{2\eta}\left(I(\pi_t^* \mid \pi_t) - I(\pi_t^g \mid \pi_t)\right).
\end{align*}
Since $K_t$ is linear, concavity of $R$ implies concavity of $V(t, \cdot)$.
Applying Lemma~\ref{lem:regularized_functional_gap} with $\mu = \pi_t^g$ gives
\begin{align*}
    \kl(\pi_t^g \mid \pi_t^*) \leq \eta\Delta_t \leq \kl(\pi_t^g \mid \nu_t^g).
\end{align*}
Thus, LSI$(\alpha_t)$ yields
\begin{align*}
    I(\pi_t^g \mid \nu_t^g) \geq 2\eta\alpha_t\Delta_t.
\end{align*}
Applying Lemma~\ref{lem:fisher_information_comparison_pinsker} with $\mu = \pi_t^g$ gives
\begin{align*}
    I(\pi_t^* \mid \pi_t) - I(\pi_t^g \mid \pi_t) & \leq \frac{5C_t}{\sqrt{2}} \sqrt{\kl(\pi_t^g \mid \pi_t^*)} \\
                                                  & \leq 5C_t\sqrt{\frac{\eta}{2}}\sqrt{\Delta_t}.
\end{align*}
Consequently,
\begin{align*}
    \frac{\d}{\d t}\Delta_t
    \leq -2a_t\Delta_t + 2a_th_t\sqrt{\Delta_t},
\end{align*}
where
\begin{align*}
    a_t                     & := \frac{\lambda\sigma_t^2\alpha_t}{\eta},
                            & h_t                                        & := \frac{5\sqrt{\eta}C_t}{4\sqrt{2}\lambda\alpha_t}, \\
    A_{\lambda,\varepsilon} & := \int_\varepsilon^1 a_s \d s,
                            & w_\lambda(t)                               & := a_t\exp\left(-\int_t^1 a_s \d s\right).
\end{align*}
For $\delta>0$, let $u_{\delta,t}:=\sqrt{\Delta_t+\delta}$. Then,
\begin{align*}
    \frac{\d}{\d t}u_{\delta,t}
    \leq -a_tu_{\delta,t}+a_th_t+a_t\sqrt{\delta}.
\end{align*}
Applying Gr\"onwall's inequality and sending $\delta\downarrow0$ implies
\begin{align*}
    \sqrt{\Delta_1}
    \leq e^{-A_{\lambda,\varepsilon}}\sqrt{\Delta_\varepsilon}
    + \int_\varepsilon^1 h_t w_\lambda(t) \d t.
\end{align*}
Moreover,
\begin{align*}
    \int_\varepsilon^1 w_\lambda(t) \d t
    = 1-e^{-A_{\lambda,\varepsilon}}.
\end{align*}
Hence, $e^{-A_{\lambda,\varepsilon}}\delta_\varepsilon + w_\lambda(t)\d t$ is a probability measure on $\{\varepsilon\}\cup[\varepsilon,1]$. Jensen's inequality gives
\begin{align*}
    \Delta_1
     & \leq e^{-A_{\lambda,\varepsilon}}\Delta_\varepsilon
    + \int_\varepsilon^1 h_t^2 w_\lambda(t) \d t                 \\
     & = e^{-A_{\lambda,\varepsilon}}\Delta_\varepsilon
    + \frac{25\eta}{32\lambda^2}
    \int_\varepsilon^1 \frac{C_t^2}{\alpha_t^2}w_\lambda(t) \d t \\
     & \leq O(\varepsilon)
    + \frac{25\eta}{32\lambda^2}
    \int_\varepsilon^1 \frac{C_t^2}{\alpha_t^2}w_\lambda(t) \d t,
\end{align*}
where the initialization term is in fact $O(\varepsilon^2)$ for the linear flow interpolation.
Finally, since $K_1$ is the identity kernel, $\mathcal V(1,\mu)=J_\eta(\mu)$ and therefore
\begin{align*}
    \Delta_1 = \sup_{\mu\in\mathcal P}J_\eta(\mu)-J_\eta(\pi_1^g),
\end{align*}
which proves the claim.

\section{Experimental Details and Additional Results}\label{sec:experimental_details}

\subsection{Variance Reduction}\label{sec:variance_reduction}
Equations~\eqref{eq:steepest_estimator} and~\eqref{eq:steepest_estimator_nonlinear} state the uncentered estimators used in the theoretical discussion.
In all experiments, we center the coefficients multiplying the conditional score to reduce variance and remove the finite-sample dependence on the additive constant of the first-order variation.

For Steepest Guidance, let
\begin{align*}
    \psi_i^{(j)}
    := \fdv{R}{\mu}[\hat\mu_{1,l}](y_{l,i}^{(j)}),
    \qquad
    b_{-i}^{(j)}
    := \frac{1}{k-1}\sum_{i'\neq i}\psi_{i'}^{(j)},
\end{align*}
where $k\geq2$ and $b_{-i}^{(j)}$ is the leave-one-out baseline.
The estimator used in the experiments is
\begin{align*}
    \hat g_{t_l}^{\mathrm{steepest,LOO}}(y_l^{(j)})
    = \lambda\sigma_{t_l}^2\frac{1}{k}\sum_{i=1}^k
    \left(\psi_i^{(j)}-b_{-i}^{(j)}\right)
    \grad_{y_l^{(j)}}\log\pi_1(y_{l,i}^{(j)}\mid Y_{t_l}=y_l^{(j)}).
\end{align*}
This estimator is exactly invariant to replacing every $\psi_i^{(j)}$ by $\psi_i^{(j)}+c$ for an arbitrary constant $c$.
For a linear reward functional, the leave-one-out baseline is independent of the $i$-th lookahead sample conditional on $Y_{t_l}=y_l^{(j)}$ and therefore preserves the unbiasedness of the estimator because the conditional score has zero expectation.
For non-linear reward functionals, the finite-particle estimator remains generally biased as discussed in the main text; the leave-one-out centering is used for invariance and variance reduction.

For the softmax estimator used by DOIT, the coefficients have empirical mean $1/k$ by construction, where $k$ is the number of lookahead samples.
We therefore subtract $1/k$ from each softmax coefficient before multiplying it by the conditional score.
Since the conditional score has zero expectation, this centering leaves the target expectation unchanged while removing the constant component of the coefficient and reducing variance.

\subsection{Experimental Setup for Bias Evaluation.}
\label{sec:bias_experimental_setup}
Here, we describe the experimental setup of the bias evaluation.
We evaluate the finite-sample bias in Fig.~\ref{fig:bias} on a one-dimensional
Gaussian toy problem. The terminal distribution is $X_1\sim\mathcal N(0,1)$,
and we evaluate the guidance at $t=0.5$.
We set the linear reward $r(x)=x$ and $\lambda=5$.
In this toy problem, we can sample the conditional distribution exactly.

For each of 128 independently sampled states and each
$k\in\{1,2,4,8,16,32\}$, we draw 512 independent REINFORCE estimates. Let
$g^*$ denote the analytic Doob's guidance, $\bar g_k$ the empirical one,
and $s_k^2$ the unbiased empirical variance across repeats. We report the
relative bias
\begin{align*}
    \frac{\sqrt{\Expec[y]{(\bar g_k-g^*)^2-s_k^2/512}}}
    {\sqrt{\Expec[y]{(g^*)^2}}},
\end{align*}
where the expectations are empirical averages across states.
The subtraction in the numerator removes the finite-repeat variance contribution to the squared
error of $\bar g_k$. Error bars show the standard error across states.

Finally, we show that the plug-in estimator has no bias in this Gaussian linear-reward setting.
The conditional sample can be expressed as follows:
\begin{align*}
    X_i(y)=a_t y+b_t \epsilon_i, \qquad \epsilon_i\sim\mathcal N(0,1),
\end{align*}
where $a_t$ and $b_t$ are functions of $t$ but not of $y$.
Then, the plug-in estimator is
\begin{align*}
    \hat g^{\mathrm{plug}}_k(y)
     & = \sigma_t^2 \grad_y \log\left(\frac{1}{k}\sum_{i=1}^k
    \exp\bigl(\lambda X_i(y)\bigr)\right)                     \\
     & = \sigma_t^2 \grad_y\left[\lambda ay+
        \log\left(\frac{1}{k}\sum_{i=1}^k
        \exp(\lambda b\epsilon_i)\right)\right] = \sigma_t^2 \lambda a.
\end{align*}
On the other hand, the analytic Doob's guidance is
\begin{align*}
    g^*(y) & = \sigma_t^2 \grad_y\log\Expec{
    \exp\bigl(\lambda(ay+b\epsilon)\bigr)}                                      \\
           & = \sigma_t^2 \grad_y\left(\lambda ay+\frac{\lambda^2b^2}{2}\right)
    = \sigma_t^2 \lambda a.
\end{align*}
Thus, $\hat g^{\mathrm{plug}}_k(y)=g^*(y)$, which implies that the plug-in estimator has no bias in this setting.

\subsection{Reward Functions.}
\paragraph{ImageReward and PickScore.}
ImageReward~\citep{xu2023imagereward} and PickScore~\citep{kirstain2023pick} are reward models that predict the score of an image based on human preferences.
\paragraph{Blueness.}
Following~\citep{dandapanthula2026we}, we define the blueness reward function as
\begin{align*}
    r_{\mathrm{blue}}(x) = \frac{1}{HW} \sum_{i=1}^{H} \sum_{j=1}^{W} (B_{i,j} - R_{i, j} - G_{i, j}),
\end{align*}
where $H$ and $W$ are the height and width of the image, respectively, and $R_{i,j}$, $G_{i,j}$, and $B_{i,j}$ are the red, green, and blue pixel values at position $(i,j)$, respectively.
\paragraph{Compressibility.}
Compressibility is defined as
\begin{align*}
    r_{\mathrm{compress}}(x) = -\frac{\mathrm{JPEGSize}_{q=95}(x)}{10000},
\end{align*}
where $\mathrm{JPEGSize}_{q=95}(x)$ is the size of the JPEG-compressed image $x$ with quality factor $q=95$.
This is an example of a non-differentiable reward function.

\subsection{Detailed Setup for Text-to-Image Experiments.}
\paragraph{Hyperparameter Settings.}
We set $L = 50$ for FLUX and $L=160$ for SD and sample 32 images with batch size $M=8$ for each method in the main experiments.
We summarize the grid for hyperparameter search and optimal $\lambda$ in Table~\ref{tab:t2i_hyperparameters}.

\paragraph{Guidance Schedule.}
For ImageReward and PickScore, we apply the guidance from $l = 10$ to $l = 50$ for FLUX and from $l = 120$ to $l = 160$ for SD, where $l$ is the diffusion step index.
For Blueness and Compressibility, we apply the guidance from $l = 0$ to $l = 10$ for FLUX and $l=0$ to $l=50$ for SD since the rewards are based on global structure acquired early in the process.
For SD, we apply the guidance every 5 steps to reduce the computational cost.

\paragraph{Approximate Posterior Sampling.}
Following~\citet{dandapanthula2026we}, we use the Diamond Map~\citep{holderrieth2026diamond} for approximate posterior sampling
of FLUX. Pretrained weights are available at \url{https://huggingface.co/gabeguofanclub/flux-1-dev-flowmap-lsd}.
For SD, we cannot find a pretrained Diamond Map, so
we utilize Tweedie's formula instead for 1-step denoising.
The noise ratio is set to $5.0$ for both models.

\begin{table*}[t]
    \centering
    \caption{A summary of the grid search for the guidance scale $\lambda$ in the text-to-image experiments.
        The table shows the search grid and selected optimal $\lambda$ for each prompt, reward function, model, and method.
        Lion, Tokyo, and Lake denote the prompts in Tables~\ref{tab:t2i_reward_summary}, \ref{tab:t2i_reward_summary_tokyo}, and~\ref{tab:t2i_reward_summary_lake}, respectively.}
    \label{tab:t2i_hyperparameters}
    \begin{tabular}{llrrrr}
\toprule
 & & \multicolumn{2}{c}{SD} & \multicolumn{2}{c}{FLUX} \\
\cmidrule(lr){3-4} \cmidrule(lr){5-6}
Reward & Grid & Steepest & Doob & Steepest & Doob \\
\midrule
\multicolumn{6}{l}{\textit{Lion}} \\
Blueness & [50, 100, 200, 300, 400, 500] & 500 & 300 & 500 & 400 \\
ImageReward & [20, 30, 40, 50, 60, 70] & 50 & 30 & 40 & 30 \\
PickScore & [20, 30, 40, 50, 60, 70] & 40 & 40 & 30 & 40 \\
Compressibility & [10, 20, 40, 60, 80, 100] & 100 & 20 & 100 & 100 \\
\midrule
\multicolumn{6}{l}{\textit{Tokyo}} \\
Blueness & [50, 100, 200, 300, 400, 500] & 500 & 400 & 400 & 200 \\
ImageReward & [20, 30, 40, 50, 60, 70] & 70 & 70 & 40 & 60 \\
PickScore & [20, 30, 40, 50, 60, 70] & 30 & 60 & 20 & 50 \\
Compressibility & [10, 20, 40, 60, 80, 100] & 100 & 10 & 40 & 10 \\
\midrule
\multicolumn{6}{l}{\textit{Lake}} \\
Blueness & [50, 100, 200, 300, 400, 500] & 500 & 400 & 500 & 500 \\
ImageReward & [20, 30, 40, 50, 60, 70] & 70 & 60 & 50 & 70 \\
PickScore & [20, 30, 40, 50, 60, 70] & 40 & 50 & 20 & 30 \\
Compressibility & [10, 20, 40, 60, 80, 100] & 100 & 80 & 40 & 100 \\
\bottomrule
\end{tabular}

\end{table*}

\subsection{Additional Results on Text-to-Image Generation.}
In addition to the main results reported in Table~\ref{tab:t2i_reward_summary}, we provide results for two additional prompts in Tables~\ref{tab:t2i_reward_summary_tokyo} and~\ref{tab:t2i_reward_summary_lake}.
We observe that the proposed method achieves better or comparable performance across the additional prompts, demonstrating its robustness.

\begin{table*}[t]
    \centering
    \caption{Text-to-image reward summary for the prompt ``A cinematic night photograph of a rainy Tokyo street with neon reflections.'' Results are reported as mean $\pm$ standard deviation over $n=32$ generated samples.}
    \label{tab:t2i_reward_summary_tokyo}
    \begin{tabular}{@{}llrrrr@{}}
\toprule
Model & Reward & Unguided & Steepest (Ours) & DOIT & SVDD \\
\midrule
SD & Blueness & $-0.34 \pm 0.07$ & $\mathbf{0.15} \pm 0.21$ & $-0.31 \pm 0.07$ & $-0.28 \pm 0.07$ \\
 & ImageReward & $1.32 \pm 0.29$ & $\mathbf{1.78} \pm 0.14$ & $1.55 \pm 0.24$ & $1.75 \pm 0.15$ \\
 & PickScore & $22.12 \pm 0.47$ & $22.90 \pm 0.48$ & $22.46 \pm 0.50$ & $\mathbf{22.95} \pm 0.51$ \\
 & Compressibility & $-16.35 \pm 1.73$ & $\mathbf{-4.61} \pm 1.99$ & $-15.98 \pm 1.77$ & $-15.60 \pm 1.67$ \\
\midrule
FLUX & Blueness & $-0.22 \pm 0.04$ & $\mathbf{-0.03} \pm 0.07$ & $-0.17 \pm 0.04$ & $-0.11 \pm 0.03$ \\
 & ImageReward & $0.93 \pm 0.52$ & $\mathbf{1.84} \pm 0.07$ & $1.41 \pm 0.48$ & $1.56 \pm 0.34$ \\
 & PickScore & $22.23 \pm 0.47$ & $\mathbf{23.53} \pm 0.32$ & $22.90 \pm 0.49$ & $23.30 \pm 0.34$ \\
 & Compressibility & $-8.12 \pm 1.12$ & $\mathbf{-5.13} \pm 0.27$ & $-7.15 \pm 0.76$ & $-6.35 \pm 0.57$ \\
\bottomrule
\end{tabular}

\end{table*}

\begin{table*}[t]
    \centering
    \caption{Text-to-image reward summary for the prompt ``A dreamy watercolor painting of a mountain lake at sunrise.'' Results are reported as mean $\pm$ standard deviation over $n=32$ generated samples.}
    \label{tab:t2i_reward_summary_lake}
    \begin{tabular}{@{}llrrrr@{}}
\toprule
Model & Reward & Unguided & Steepest (Ours) & DOIT & SVDD \\
\midrule
SD & Blueness & $-0.47 \pm 0.10$ & $\mathbf{0.32} \pm 0.20$ & $-0.44 \pm 0.10$ & $-0.40 \pm 0.09$ \\
 & ImageReward & $0.03 \pm 0.43$ & $\mathbf{0.64} \pm 0.16$ & $0.33 \pm 0.26$ & $0.58 \pm 0.18$ \\
 & PickScore & $21.84 \pm 0.54$ & $22.54 \pm 0.52$ & $22.12 \pm 0.53$ & $\mathbf{22.57} \pm 0.52$ \\
 & Compressibility & $-11.76 \pm 1.42$ & $\mathbf{-4.21} \pm 0.96$ & $-11.41 \pm 1.44$ & $-10.98 \pm 1.59$ \\
\midrule
FLUX & Blueness & $-0.58 \pm 0.07$ & $\mathbf{0.08} \pm 0.09$ & $-0.48 \pm 0.07$ & $-0.38 \pm 0.07$ \\
 & ImageReward & $0.97 \pm 0.21$ & $\mathbf{1.42} \pm 0.12$ & $1.17 \pm 0.18$ & $1.34 \pm 0.17$ \\
 & PickScore & $23.44 \pm 0.49$ & $\mathbf{24.37} \pm 0.36$ & $24.02 \pm 0.38$ & $24.32 \pm 0.45$ \\
 & Compressibility & $-7.78 \pm 1.12$ & $\mathbf{-4.66} \pm 0.29$ & $-6.83 \pm 1.03$ & $-5.54 \pm 0.48$ \\
\bottomrule
\end{tabular}

\end{table*}

\subsection{Examples of Generated Images.}
Here, we provide examples of images generated with SD and FLUX for the Blueness, ImageReward, PickScore, and Compressibility reward functions.
\begin{figure}
    \centering
    \includegraphics[width=\textwidth]{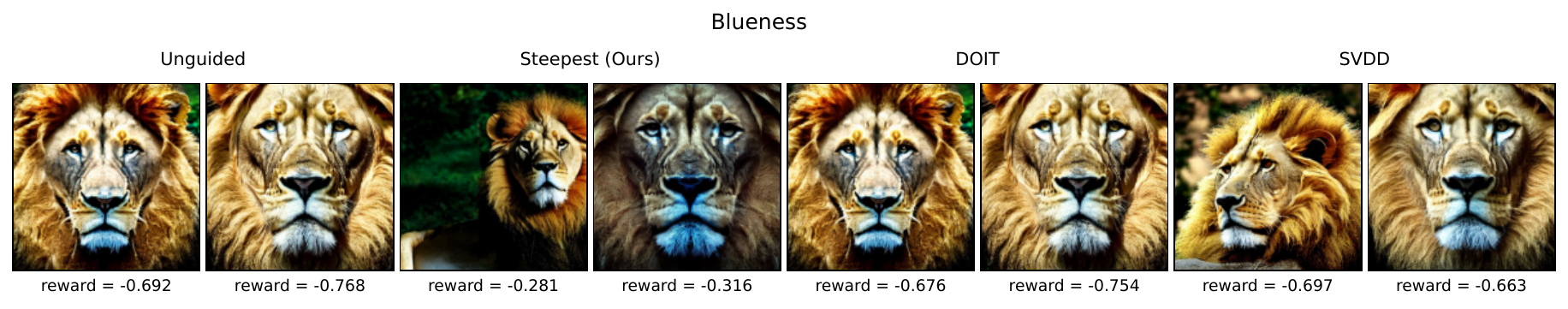}
    \includegraphics[width=\textwidth]{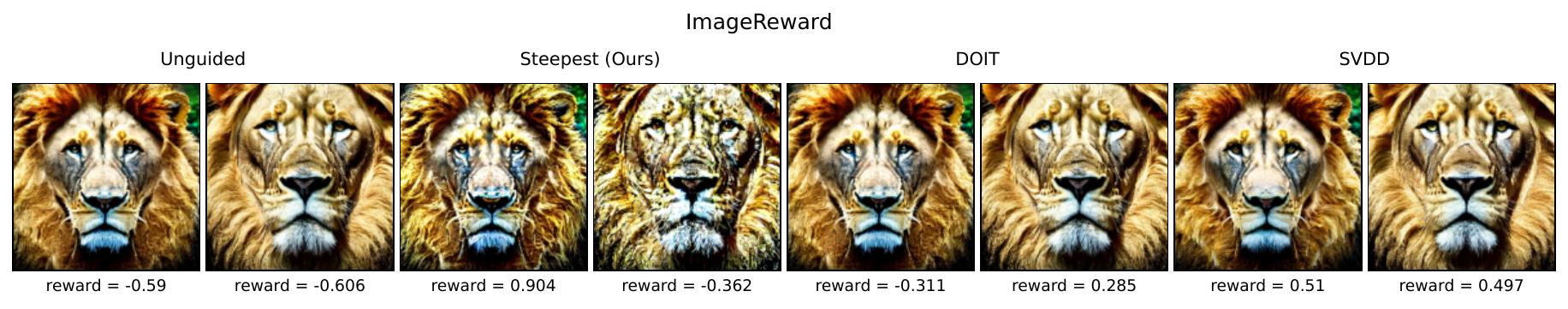}
    \includegraphics[width=\textwidth]{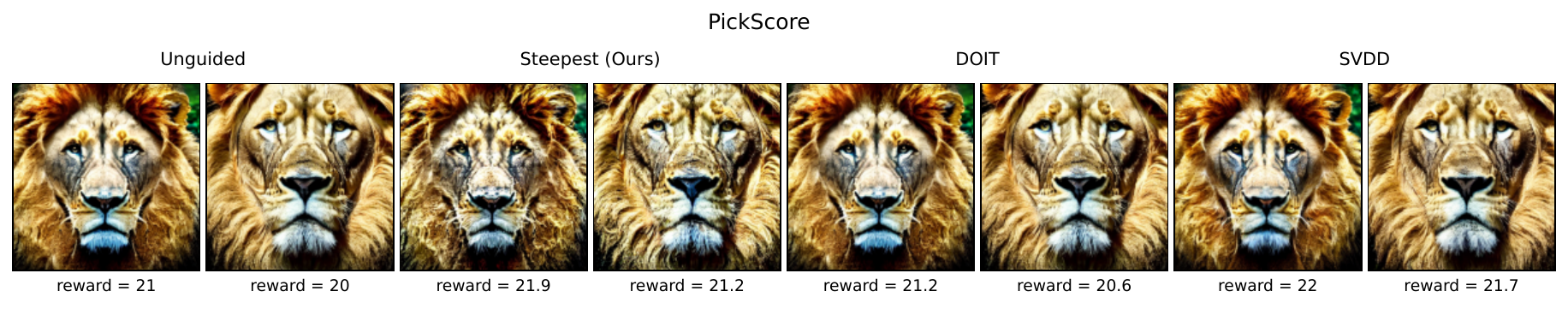}
    \includegraphics[width=\textwidth]{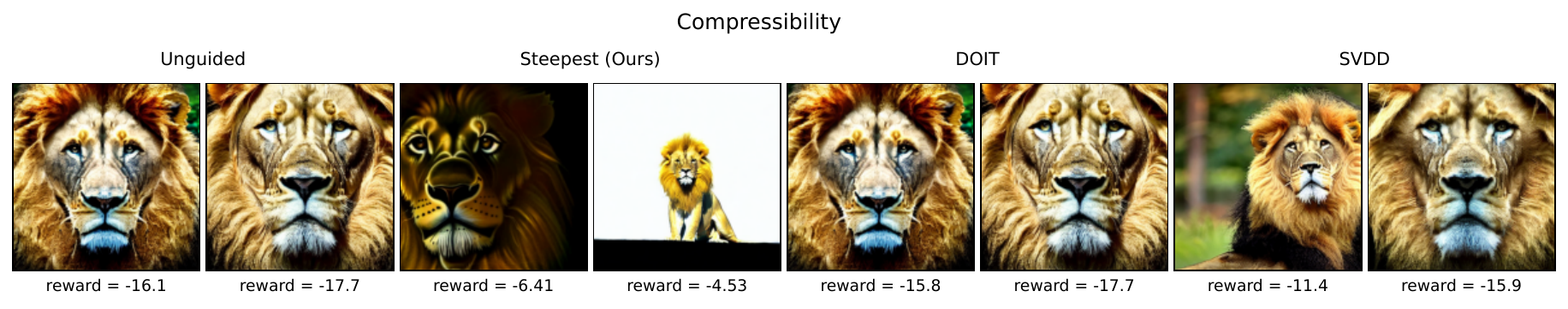}
    \captionof{figure}{Selected images generated with SD for the Blueness, ImageReward, PickScore, and Compressibility reward functions (top to bottom).}
    \label{fig:selected_images_sd}
\end{figure}

\begin{figure}
    \centering
    \includegraphics[width=\textwidth]{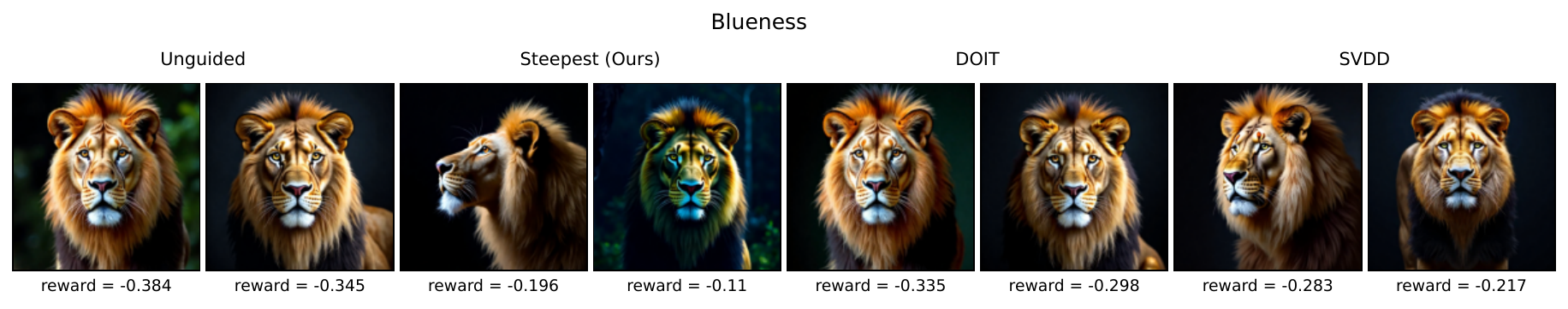}
    \includegraphics[width=\textwidth]{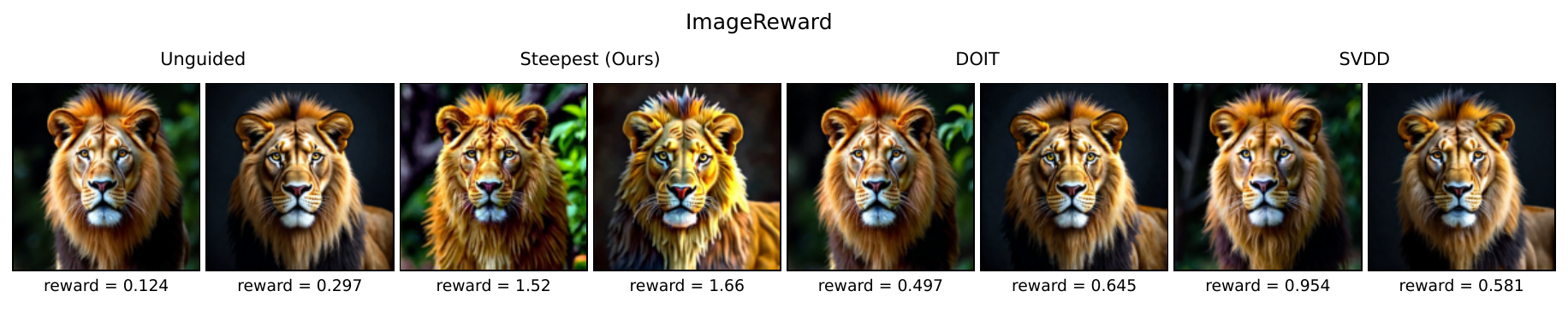}
    \includegraphics[width=\textwidth]{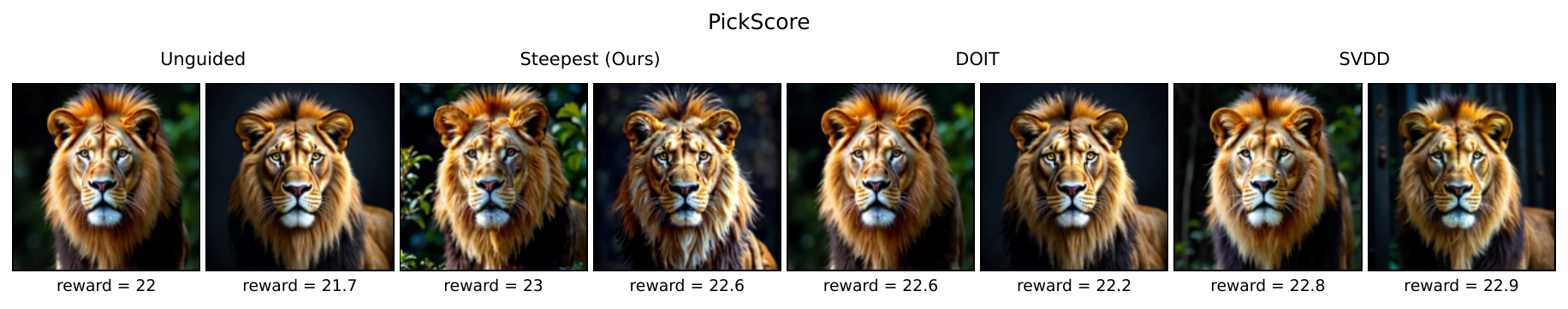}
    \includegraphics[width=\textwidth]{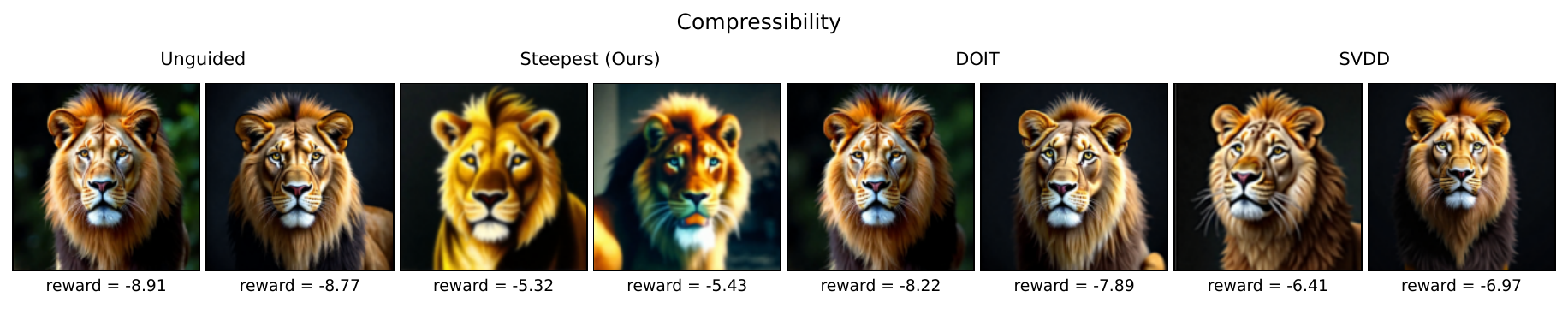}
    \captionof{figure}{Selected images generated with FLUX for the Blueness, ImageReward, PickScore, and Compressibility reward functions (top to bottom).}
    \label{fig:selected_images_flux}
\end{figure}

\subsection{Additional Results on Non-linear Reward Functionals.}
\paragraph{Diversity Seeking Generation.}
To enhance diversity in generated images, we compare entropy and Rao's quadratic entropy as reward functionals.
For Rao's quadratic entropy, we use DINOv2 to extract features and an RBF kernel to compute the entropy.
Note that entropy regularization does not require an additional model for feature extraction.
The bandwidth of the RBF kernel is adaptively set to the median of pairwise squared distances between features divided by $\max\{\log M, 1\}$.
We set the guidance scale $\lambda$ to $1,000$ and Rao's quadratic entropy reward weight to $1$.
For entropy maximization, we use the KL-regularized formulation described in Section~\ref{sec:examples_reward_functionals} and set its reward and KL weights to $2\times10^{-4}$.
Figure~\ref{fig:t2i_sd_kernel_weight_comparison} shows images generated for the prompt ``VAN GOGH CAFE TERASSE'' using Rao's quadratic entropy, entropy regularization, and unguided sampling.
While unguided sampling produces images with similar scene composition (e.g., the 1st, 3rd, 4th, and 7th images in the bottom row), regularized sampling produces visually varied images.
We also report in-batch diversity scores in Table~\ref{tab:t2i_sd_kernel_dino_diversity}.
The score is computed as the mean pairwise cosine distance between the normalized DINOv2 image embeddings within the batch.
As we discuss in Section~\ref{sec:related_work}, several methods~\citep{corso2023particle,vinograd2026diverse,zilberstein2024repulsive} have been proposed to enhance diversity in generated images
but they are based on heuristics or require first- and second-order derivatives of the kernel.
\begin{table}[h]
    \centering
    \caption{Within-batch diversity for the eight images in Fig.~\ref{fig:t2i_sd_kernel_weight_comparison}, measured as mean pairwise cosine distance between normalized DINOv2 image embeddings. Higher is more diverse.}
    \begin{tabular}{lr}
    \toprule
    Method                  & DINO diversity $\uparrow$ \\
    \midrule
    Rao's Quadratic Entropy & $0.555$                   \\
    Entropy                 & $0.504$                   \\
    Unguided                & $0.428$                   \\
    \bottomrule
\end{tabular}

    \label{tab:t2i_sd_kernel_dino_diversity}
\end{table}
\begin{figure}[h]
    \centering
    \includegraphics[width=\textwidth]{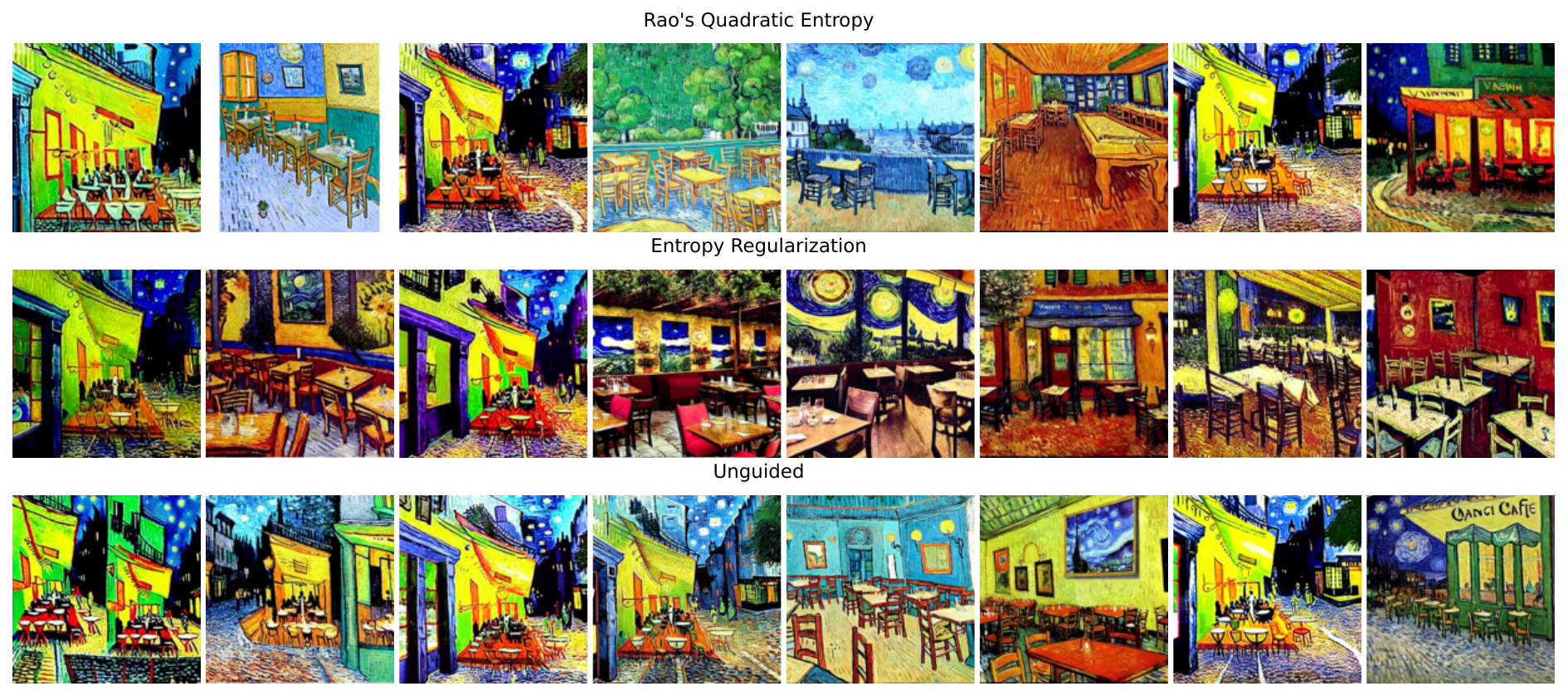}
    \captionof{figure}{Comparison of diversity-seeking generation using SD. From top to bottom, the rows show Steepest Guidance with Rao's quadratic entropy, entropy regularization, and the unguided baseline.}
    \label{fig:t2i_sd_kernel_weight_comparison}
\end{figure}

\subsection{Effect of Guidance Scaling}
Some prior works~\citep{zhu2026training,dandapanthula2026we} introduce the heuristic of using a guidance scaling factor $\gamma$ to control the strength of the guidance.
That is, the guidance is scaled by the factor $\gamma$: $\gamma g_t(y)$.
To see the effect of guidance scaling, we vary $\gamma$ for empirical Doob's h-transform and observe how it influences the resulting reward.
Here, we fix $\lambda=30$ for ImageReward and $\lambda=100$ for Compressibility, with $k=4$.
We use $\gamma\in\{2,4,6,8,10,12\}$ for ImageReward and $\gamma\in\{5,10,20,40,60,80\}$ for Compressibility.
Fig.~\ref{fig:flux_gamma_ablation} shows the results of varying the guidance scaling factor $\gamma$
and compares them to our steepest guidance.
We see that applying a guidance scaling improves the performance to some extent,
but it does not outperform steepest guidance.
One possible explanation is that guidance scaling may amplify a bias in the guidance, leading to suboptimal performance.
This suggests that the observed performance gap in the main experiments cannot be explained solely by insufficient guidance strength

\begin{figure}[h]
    \centering
    \includegraphics[width=\textwidth]{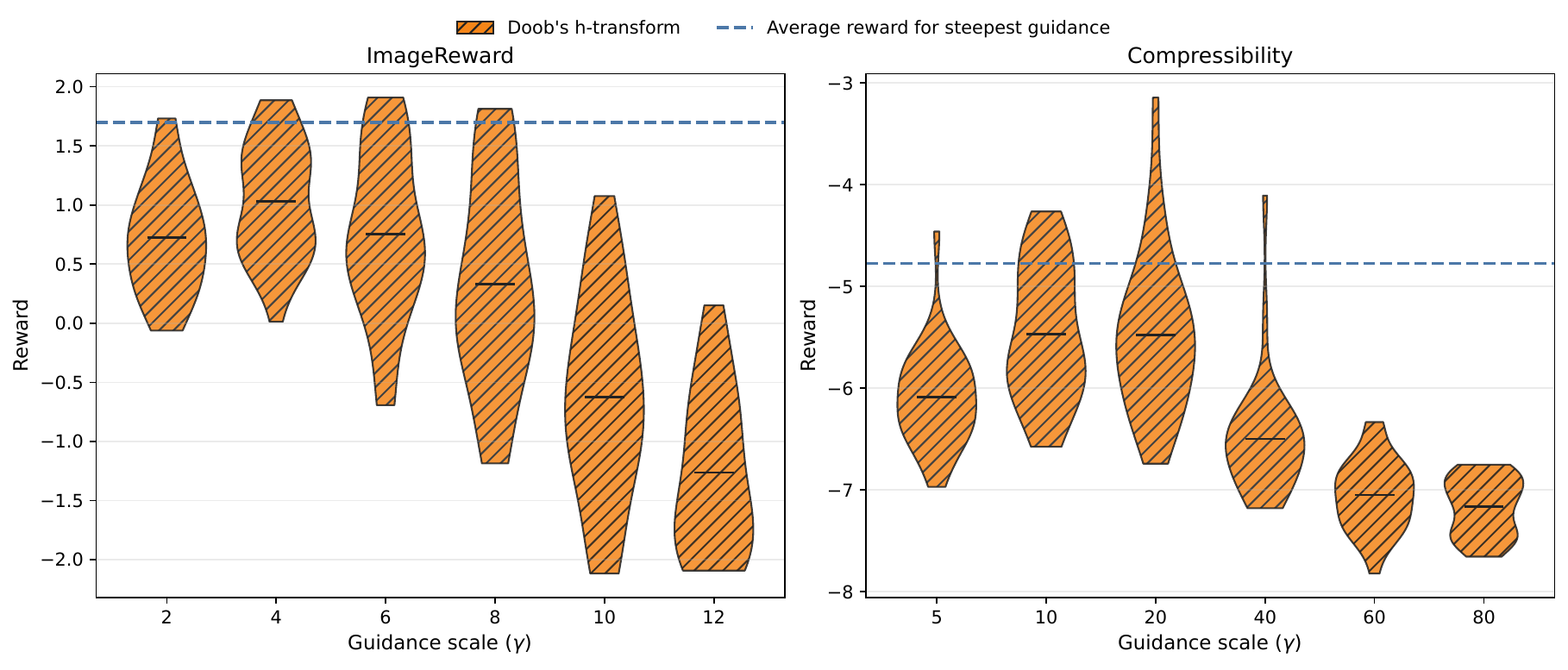}
    \captionof{figure}{Ablation study on the guidance scaling factor $\gamma$ for FLUX with ImageReward and Compressibility.
        Average performance is shown as horizontal lines.}
    \label{fig:flux_gamma_ablation}
\end{figure}

\subsection{Comparison with SALD}
Here, we compare the performance of the proposed steepest guidance method with SALD~\citep{nitanda2026slowly}.
We do not apply slowdown here to reduce the computational cost.
Following the original paper, we consider define $f_t$ in SALD as the Gaussian-smoothed reward function, i.e.,
$
    f_{t_l}(y) = \mathbb{E}_{Z \sim \mathcal{N}(0, I)}[r(y + \bar \sigma_{t_l} Z)]
$ for $\bar \sigma_{t_l} = \sqrt{\Delta_{t_l}}\sigma_t$.
and gradients are estimated using $\nabla f_{t_l}(y) \simeq \frac{1}{k\bar \sigma_{t_l}} \sum_{i=1}^{k} r(y + \bar \sigma_{t_l} Z_i) Z_i$, where $Z_i \sim \mathcal{N}(0, I)$ independently.
Then, the guidance in SALD is defined as $\lambda \sigma_t^2 \hat \nabla f_{t_l}(y)$.
For a fair comparison, we subtract the leave-one-out baseline from the reward.
Fig.~\ref{fig:flux_sald_lambda_ablation} shows the results of varying the hyperparameter $\lambda$ in SALD and compares them to our steepest guidance.
We see that SALD cannot achieve the same level of performance as our steepest guidance method.
In particular, the difference is large for ImageReward experiments.
This may be because ImageReward is sensitive to fine-grained features that are lost through Gaussian smoothing.

\begin{figure}[h]
    \centering
    \includegraphics[width=\textwidth]{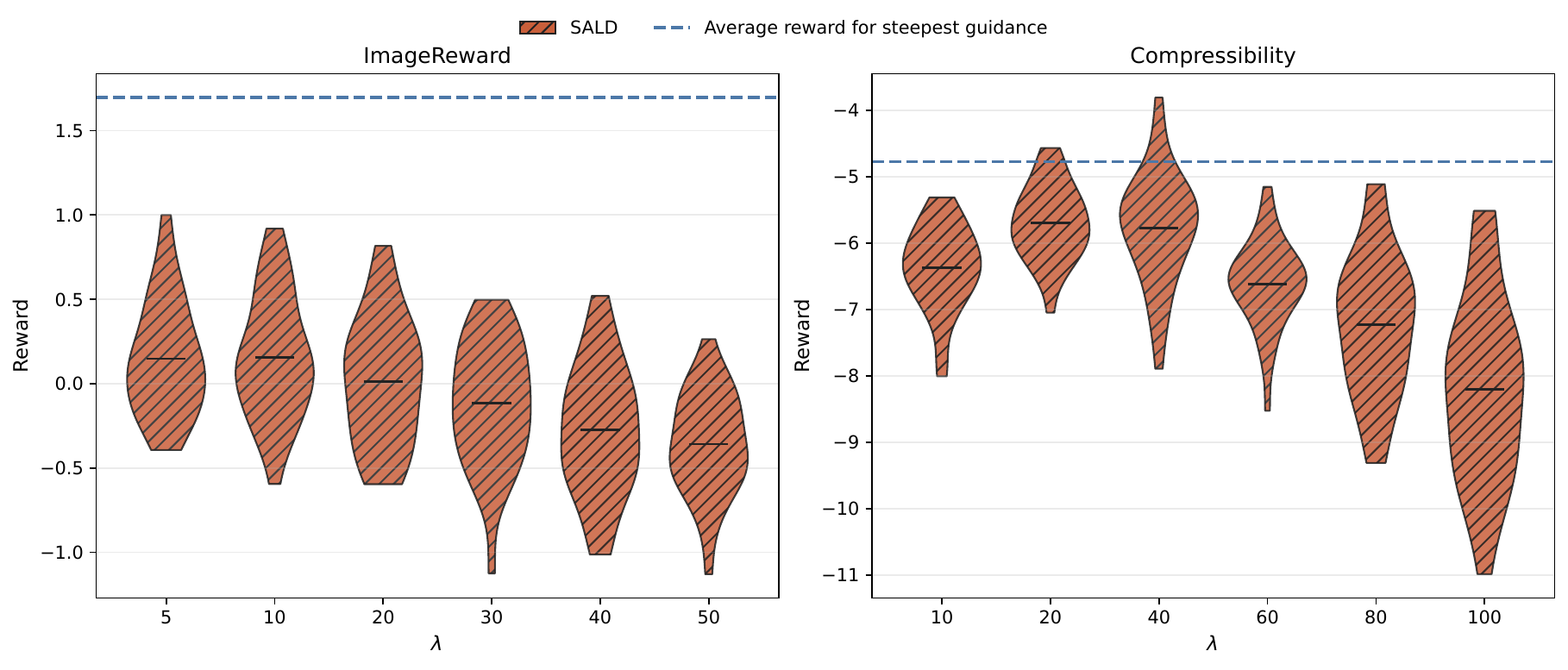}
    \captionof{figure}{Comparison of the proposed steepest guidance method with SALD~\citep{nitanda2026slowly} for FLUX with ImageReward and Compressibility.
        Average performance is shown as horizontal lines.}
    \label{fig:flux_sald_lambda_ablation}
\end{figure}

\subsection{Ablation Study on $k$.}
We conduct an ablation study on the number of lookahead particles $k$.
We vary $k\in\{4,8,16,32\}$ while keeping $\lambda$ fixed.
We use the same $\lambda$ values as in the main experiments.
Fig.~\ref{fig:flux_k_ablation} shows that the performance improves as $k$ increases, but the proposed method consistently outperforms the baselines for all $k$ values.
\begin{figure}[h]
    \centering
    \includegraphics[width=\textwidth]{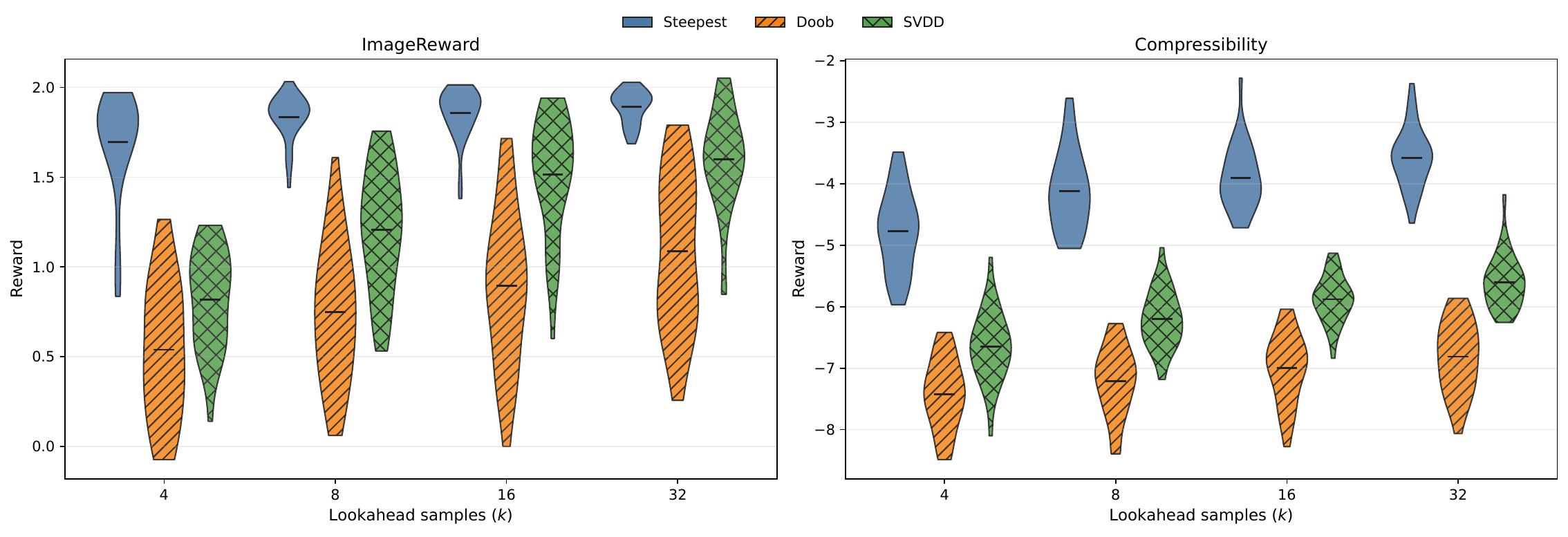}
    \captionof{figure}{Ablation study on the number of lookahead particles $k$ for FLUX with ImageReward and Compressibility.
        Average performance is shown as horizontal lines.}
    \label{fig:flux_k_ablation}
\end{figure}

\end{document}